\documentclass[11pt]{article}

\usepackage[
  top=2.54cm,
  bottom=2.54cm,
  left=2.54cm,
  right=2.54cm
]{geometry}

\usepackage{amsmath}
\usepackage{amssymb}

\usepackage{graphicx}
\usepackage{subcaption}
\usepackage{float}      
\usepackage{placeins}   

\usepackage{booktabs}
\usepackage{multirow}
\usepackage{array}

\usepackage[dvipsnames]{xcolor}

\definecolor{pwmBlue}{HTML}{2E5FA1}
\definecolor{pwmOrange}{HTML}{E8712B}
\definecolor{pwmGreen}{HTML}{3D9E3F}
\definecolor{pwmRed}{HTML}{C93C3C}
\definecolor{pwmPurple}{HTML}{7B4EA3}
\definecolor{pwmGray}{HTML}{6B6B6B}

\definecolor{gateOne}{HTML}{4A90D9}    
\definecolor{gateTwo}{HTML}{E8A838}    
\definecolor{gateThree}{HTML}{D94A4A}  

\definecolor{scenarioI}{HTML}{2E7D32}    
\definecolor{scenarioII}{HTML}{C62828}   
\definecolor{scenarioIII}{HTML}{1565C0}  
\definecolor{scenarioIV}{HTML}{6A1B9A}   

\usepackage[
  colorlinks=true,
  linkcolor=pwmBlue,
  citecolor=pwmBlue,
  urlcolor=pwmBlue
]{hyperref}
\usepackage[capitalise,noabbrev]{cleveref}

\usepackage[numbers,sort&compress,super]{natbib}

\usepackage{amsthm}
\newtheorem{theorem}{Theorem}
\newtheorem{definition}{Definition}

\newcommand{\ctier}{\mathcal{C}_{\mathrm{img}}}
\newcommand{\Blib}{\mathcal{B}}
\newcommand{\etier}{e_{\mathrm{img}}}
\newcommand{\Nmax}{N_{\max}}
\newcommand{\Dmax}{D_{\max}}
\newcommand{\Hdag}{H_{G}}

\newcommand{\pwm}{\textsc{pwm}}
\newcommand{\triad}{\textsc{Triad Decomposition}}
\newcommand{\opg}{\textsc{OperatorGraph}}

\newcommand{\eg}{\textit{e.g.}}
\newcommand{\etc}{\textit{etc.}}

\newcommand{\Hnom}{H_{\mathrm{nom}}}
\newcommand{\Htrue}{H_{\mathrm{true}}}
\newcommand{\xgt}{\mathbf{x}_{\mathrm{gt}}}
\newcommand{\yobs}{\mathbf{y}_{\mathrm{obs}}}
\newcommand{\psnr}{\mathrm{PSNR}}

\newcommand{\recoveryratio}{\rho}
\newcommand{\mismatchparam}{\boldsymbol{\theta}}
\newcommand{\triadreport}{\textsc{Triad\-Report}}

\newcommand{\gateone}{\textbf{Gate\,1}}
\newcommand{\gatetwo}{\textbf{Gate\,2}}
\newcommand{\gatethree}{\textbf{Gate\,3}}

\begin{document}

\title{Eleven Primitives and Three Gates:\\
The Universal Structure of Computational Imaging}

\author{
Chengshuai Yang$^{1,*}$
\and
Xin Yuan$^{2}$
}

\date{%
$^1$NextGen PlatformAI C Corp, USA.\quad
$^2$School of Engineering, Westlake University, Hangzhou, China.\\[2pt]
$^*$Correspondence: \url{integrityyang@gmail.com}%
}

\maketitle

\begin{abstract}
Computational imaging systems---from coded-aperture cameras to cryo-electron microscopes---span five carrier families yet share a hidden structural simplicity. We prove that every imaging forward model decomposes into a directed acyclic graph over exactly 11 physically typed primitives (Finite Primitive Basis Theorem)---a sufficient and minimal basis that provides a compositional language for designing any imaging modality~\cite{yang2026fpt}. We further prove that every reconstruction failure has exactly three independent root causes: information deficiency, carrier noise, and operator mismatch (Triad Decomposition). The three gates map to the system lifecycle: Gates~1 and~2 guide design (sampling geometry, carrier selection); Gate~3 governs deployment-stage calibration and drift correction. Validation across 12 modalities and all five carrier families confirms both results, with $+0.8$ to $+13.9$\,dB recovery on deployed instruments. Together, the 11 primitives and 3 gates establish the first universal grammar for designing, diagnosing, and correcting computational imaging systems.
\end{abstract}

\section*{Introduction}

\begin{sloppypar}
A coded-aperture camera, an MRI scanner, and a cryo-electron microscope appear to share nothing: they employ different physical carriers, different geometries, and different detectors. Yet each converts a physical scene into digital measurements by composing a small chain of transformations---propagation, interaction, encoding, detection---and each reconstructs the scene by inverting a mathematical model of that chain. When the model matches reality, modern algorithms deliver extraordinary results. When it does not---and on deployed instruments, it never perfectly does---reconstruction degrades or fails entirely. The failure is silent and systematic: the algorithm faithfully solves the wrong problem. Despite a decade of progress in reconstruction algorithms, no general framework exists to explain \emph{why} a reconstruction failed, \emph{where} in the forward model the error resides, or \emph{how} to fix it without retraining the solver.
\end{sloppypar}

This paper establishes that the space of imaging forward models has a surprisingly simple and universal structure---a finite grammar of 11 primitives---and that every reconstruction failure admits a complete, actionable diagnosis through exactly 3 independent gates. Over the past decade, the community has invested enormous effort in improving reconstruction algorithms---progressing from compressed sensing~\cite{candes2008compressed,donoho2006compressed} and plug-and-play priors~\cite{venkatakrishnan2013pnp,yuan2021pnp} to deep unrolling networks~\cite{monga2021unrolling} and vision transformers~\cite{cai2022mst}---with snapshot compressive imaging emerging as a unifying framework for coded aperture systems~\cite{yuan2021sci}. These advances are real and important: a better solver on a well-calibrated operator remains the gold standard. Yet algorithms routinely fail when deployed on real instruments~\cite{antun2020instabilities}, and the root cause is often not the algorithm but an undiagnosed forward-model error. Calibration methods exist for specific instruments---ESPIRiT~\cite{uecker2014espirit} for MRI, ePIE~\cite{maiden2009ptychography} for electron ptychography---but they do not generalise across modalities, and no systematic framework decomposes reconstruction failures into root causes or guides the design of new imaging systems from first principles.

Here we establish two complementary foundations within Physics World Models (\pwm{}). The \textbf{Finite Primitive Basis Theorem} (Theorem~\ref{thm:fpb}) proves that every imaging forward model in a broad operator class $\ctier$ admits an $\varepsilon$-approximate representation as a typed directed acyclic graph (DAG) over exactly 11 canonical primitives---and that this library is minimal~\cite{yang2026fpt}. The \textbf{Triad Decomposition} proves that every reconstruction failure decomposes into three gates---information deficiency (\gateone{}), carrier noise (\gatetwo{}), and operator mismatch (\gatethree{})---each independently testable and independently addressable. The 11 primitives \emph{enable} the 3-gate diagnosis: the DAG structure localizes each gate to specific physical operators, making both design and correction actionable.

The framework serves dual purposes (\cref{fig:overview}). For \textbf{existing instruments}, the Triad identifies the dominant gate and prescribes targeted intervention---calibration for Gate~3, acquisition redesign for Gate~1, or source/detector improvement for Gate~2. For \textbf{new instruments}, the 11-primitive basis provides a compositional design language: specifying a new modality reduces to selecting primitives and their parameters, with Gate~1 analysis guiding sampling strategy, Gate~2 guiding carrier and source selection, and Gate~3 predicting the calibration accuracy required for a target reconstruction quality. Solver design remains essential throughout: once the dominant gate is addressed, a better solver delivers further gains on the corrected operator.

\begin{sloppypar}
We validate both results---the Finite Primitive Basis Theorem (every forward model decomposes into 11 primitives) and the Triad Decomposition (every reconstruction failure has exactly three root causes)---across twelve modalities spanning all five carrier families---optical photons (CASSI~\cite{wagadarikar2008cassi}, CACTI~\cite{llull2013cacti}, SPC~\cite{duarte2008spc}, lensless imaging, compressive holography, fluorescence microscopy), X-ray photons (CT~\cite{feldkamp1984practical}, CBCT), electrons (electron ptychography~\cite{rodenburg2004ptychography}, cryo-EM), nuclear spins (MRI~\cite{pruessmann1999sense,lustig2007sparse}), and acoustic waves (ultrasound)---with hardware validation on all twelve modalities confirming Triad predictions on physical measurements. The validated modalities include two Nobel Prize--winning techniques (cryo-EM, 2017 Chemistry; super-resolution fluorescence, 2014 Chemistry), three clinically deployed instruments (CT, CBCT, MRI), and the first acoustic-carrier validation in computational imaging. A held-out closure test on 8 additional modalities confirms basis completeness with saturating growth. The proof of Theorem~\ref{thm:fpb} is in Supplementary Note~12; formal primitive semantics and typed DAG denotation are developed in~\cite{yang2026fpt}.
\end{sloppypar}

\paragraph{Generalization beyond imaging.}
The finite-basis result raises a natural question: do other measurement domains admit similarly compact primitive decompositions? The \opg{} formalism and Triad Decomposition are structurally applicable wherever a forward model maps quantities of interest to observations through composable physical stages. Seismology (wave propagation $+$ attenuation $+$ sensor response), radio astronomy (aperture synthesis $+$ ionospheric distortion $+$ noise), and particle physics (interaction model $+$ detector response $+$ pile-up) all exhibit the same three-gate failure structure. Whether these domains also admit small primitive bases is an empirical question that our framework is designed to answer.

\begin{figure}[H]
\centering
\includegraphics[width=\textwidth]{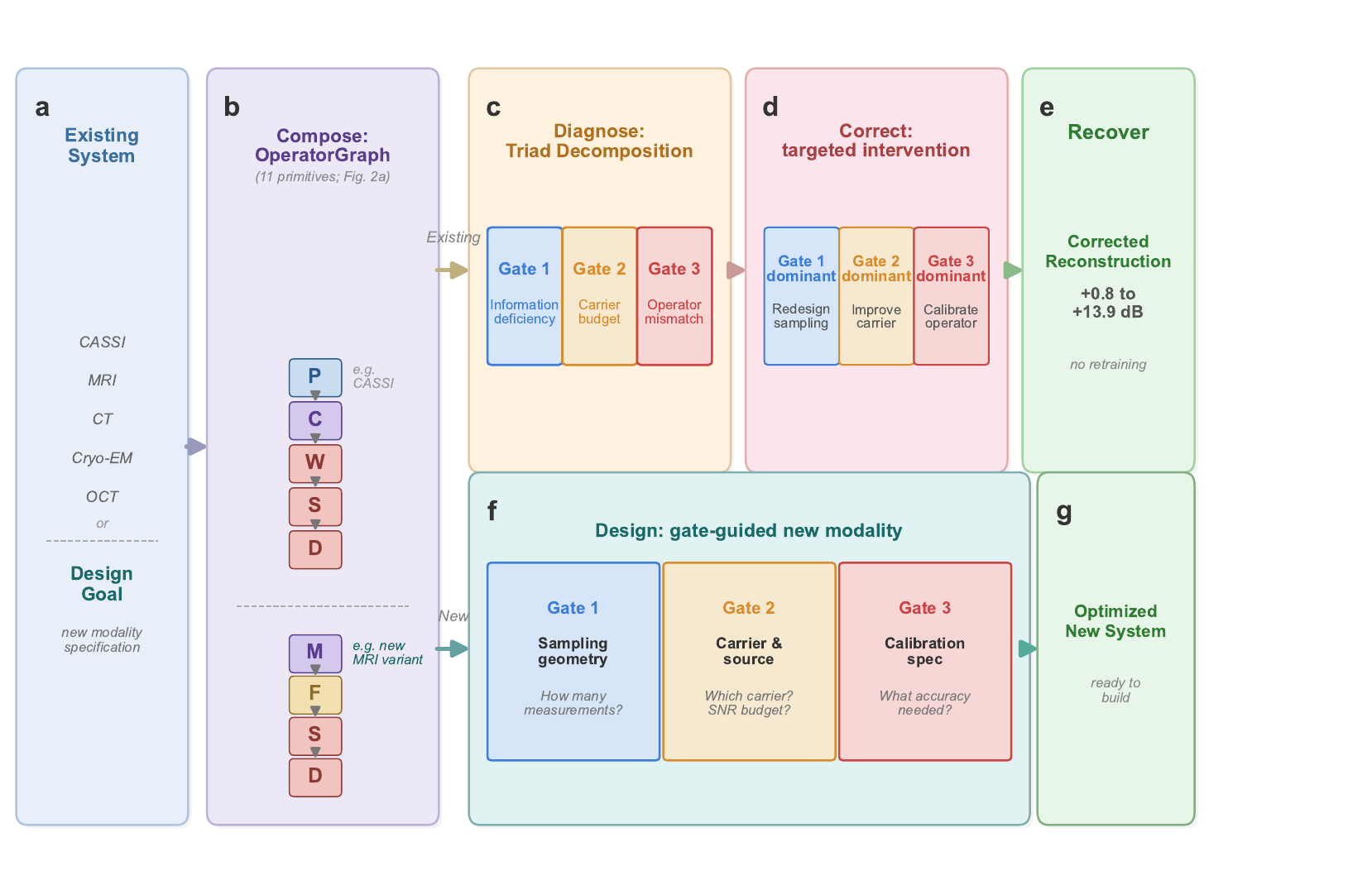}
\caption{\textbf{The universal grammar of computational imaging.}
\textbf{a},~Any computational imaging system---from coded-aperture cameras to cryo-electron microscopes---serves as input.
\textbf{b},~The system is composed as a typed directed acyclic graph (DAG) over 11 universal primitives (\cref{fig:operatorgraph}); shown here for CASSI as an example.
The grammar enables two workflows.
\emph{Top (existing instruments):}
\textbf{c},~The \triad{} diagnoses the DAG through three independent gates: information deficiency (Gate~1), carrier budget (Gate~2), and operator mismatch (Gate~3).
\textbf{d},~Based on the dominant gate, a targeted correction is applied: sampling redesign (Gate~1), source/detector improvement (Gate~2), or operator calibration (Gate~3).
\textbf{e},~The existing solver is re-run with the corrected operator, recovering $+0.8$ to $+13.9$\,dB without retraining.
\emph{Bottom (new instruments):}
\textbf{f},~The same three gates guide the design of a new modality from first principles: Gate~1 determines sampling geometry, Gate~2 guides carrier and source selection, and Gate~3 predicts the calibration accuracy required for a target reconstruction quality.
\textbf{g},~The output is an optimized new system specification ready to build.
Together, the 11 primitives (alphabet) and 3 gates (rules) form a universal grammar for designing, diagnosing, and correcting any imaging system.}
\label{fig:overview}
\end{figure}

\FloatBarrier
\section*{The Finite Primitive Basis}

A foundational question in computational imaging is whether the diversity of imaging forward models---from coded aperture cameras to MRI scanners to electron microscopes---can be captured by a small, fixed set of primitive operators. We answer affirmatively.

\paragraph{The OperatorGraph representation.}
Every imaging forward model is encoded as a typed directed acyclic graph (DAG) in which each node wraps a single primitive physical operator and edges define the data flow from source to detector. Every primitive implements both a \texttt{forward()} and an \texttt{adjoint()} method, with validated adjoint consistency ensuring $\langle H\mathbf{x}, \mathbf{y} \rangle = \langle \mathbf{x}, H^{\dagger}\mathbf{y} \rangle$ to within numerical precision. Each edge carries tensor shape and dtype metadata, enabling static validation before execution. We call this formalism the \opg{} intermediate representation (IR).

\paragraph{Eleven canonical primitives.}
The \opg{} IR defines a library of exactly 11 canonical primitives $\Blib = \{P, M, \Pi, F, C, \Sigma, D, S, W, R, \Lambda\}$:

\begin{table}[h]
\centering\small
\begin{tabular}{@{}clll@{}}
\toprule
\# & Primitive & Notation & Physical action \\
\midrule
1 & Propagate  & $P(d,\lambda)$ & Free-space wave propagation \\
2 & Modulate   & $M(\mathbf{m})$ & Element-wise multiplication (mask, coil, absorption) \\
3 & Project    & $\Pi(\theta)$ & Radon line-integral projection \\
4 & Encode     & $F(\mathbf{k})$ & Fourier-domain encoding ($k$-space) \\
5 & Convolve   & $C(\mathbf{h})$ & Spatial convolution (PSF) \\
6 & Accumulate & $\Sigma$ & Summation over spectral/temporal axis \\
7 & Detect     & $D(g,\eta)$ & Detector response (5 canonical families) \\
8 & Sample     & $S(\Omega)$ & Sub-sampling on index set $\Omega$ \\
9 & Disperse   & $W(\alpha,a)$ & Wavelength-dependent spatial shift \\
10 & Scatter   & $R(\sigma,\Delta\varepsilon)$ & Direction change and/or energy shift \\
11 & Transform & $\Lambda(f,\boldsymbol{\theta})$ & Pointwise nonlinear physics operation \\
\bottomrule
\end{tabular}
\end{table}

\noindent The Detect nonlinearity $\eta$ is restricted to five canonical families (linear-field: $\eta(x) = gx$; logarithmic; sigmoid; intensity-square-law: $\eta(x) = g|x|^2$; coherent-field), each with at most 2 scalar parameters. The Transform primitive $\Lambda$ applies a pointwise nonlinear function within the physics chain (not at the detector) and is restricted to five families: exponential attenuation ($e^{-\alpha x}$, Beer--Lambert), logarithmic compression, phase wrapping ($\arg(e^{ix})$), polynomial ($\sum a_k x^k$, $d \leq 5$), and saturation. Both Detect and Transform are pointwise and parameter-limited, preventing either from becoming a universal approximator (see~\cite{yang2026fpt} for formal definitions and non-universality proofs).

\paragraph{Physics-stage mapping.}
Each factor of a forward model is classified into one of six physics-stage families and mapped to primitives from $\Blib$: propagation $\to$ $\{P, C\}$; elastic interaction $\to$ $\{M\}$; inelastic interaction (scattering) $\to$ $\{R\}$; pointwise nonlinear physics $\to$ $\{\Lambda\}$; encoding--projection $\to$ $\{\Pi, F\}$; detection--readout $\to$ $\{\Sigma, S, W, C, D\}$. This classification formalizes the source--medium--sensor decomposition that structures classical imaging theory~\cite{brady2009optical}.

\paragraph{Fidelity levels.}
Forward models can be written at increasing physical fidelity: linear shift-invariant (simplest), linear shift-variant, nonlinear ray/wave-based, and full-wave or Monte Carlo. All 11 primitives apply uniformly across every fidelity level. Linear stages are composed from the first 10 primitives; nonlinear stages (beam hardening, phase wrapping, saturation) are captured by Transform~$\Lambda$. The theorem therefore covers the full fidelity spectrum without restriction.

\paragraph{Physical carriers.}
The template library spans five carrier families---optical photons, X-ray photons, electrons, nuclear spins, and acoustic waves---with particle-beam probes (neutrons, protons, muons) representable in the extended registry using \emph{zero new primitives}: neutron CT reuses the X-ray CT DAG ($P \!\to\! \Lambda \!\to\! D$) with nuclear cross-section parameters; proton CT inserts a Bethe--Bloch stopping-power stage via Transform~$\Lambda$ (\#11); muon tomography uses Scatter~$R$ (\#10, already introduced for Compton imaging) for multiple Coulomb scattering and POCA reconstruction (Supplementary Note~S24). In total, 170 modality templates cover 23 categories (microscopy, medical imaging, electron microscopy, remote sensing, spectroscopy, and others; see~\cite{yang2026fpt}), of which 12 now have full end-to-end correction validation across all five primary carrier families (Supplementary Table~S3).

\paragraph{Theorem statement.}

\begin{definition}[Imaging Operator Class]
\label{def:ctier}
$\ctier$ consists of all imaging forward models expressible as a finite sequential-parallel composition of stages---each either linear with bounded operator norm ($\|H_k\| \leq B$) or pointwise nonlinear with bounded Lipschitz constant ($\mathrm{Lip}(\Lambda_k) \leq L$)---with stage count $K \leq \Nmax$.
\end{definition}

\begin{definition}[$\varepsilon$-Approximate Representation]
\label{def:eps}
A typed DAG $G$ with nodes $V \subseteq \Blib$ is an $\varepsilon$-approximate representation of $H \in \ctier$ if $\sup_{\|\mathbf{x}\|\leq 1} \|H(\mathbf{x}) - \Hdag(\mathbf{x})\| / (\|H(\mathbf{x})\| + \delta) \leq \varepsilon$, where $\delta > 0$ is a regularization constant, and $|V| \leq \Nmax$, $\mathrm{depth}(G) \leq \Dmax$.
\end{definition}

\begin{theorem}[Finite Primitive Basis]
\label{thm:fpb}
For every $H \in \ctier$, there exists a typed DAG $G$ with $V \subseteq \Blib$ that is an $\varepsilon$-approximate representation of $H$.
\end{theorem}

\noindent\textit{Proof sketch.} By Definition~\ref{def:ctier}, $H = H_K \circ \cdots \circ H_1$ with $K \leq \Nmax$. Each factor is classified into one of six physics-stage families and realized by primitives: propagation $\to$ $P$ or $C$; elastic interaction $\to$ $M$; inelastic interaction $\to$ $R$; pointwise nonlinear physics $\to$ $\Lambda$; encoding--projection $\to$ $\Pi$ or $F$; detection--readout $\to$ finite composition from $\{\Sigma, S, W, C, D\}$. For linear stages, per-factor errors compose via sub-multiplicativity: $\|H - \Hdag\| \leq K \cdot \max_k(\varepsilon_k) \cdot B^{K-1}$. For nonlinear stages, the Lipschitz composition bound $\|H(\mathbf{x}) - \hat{H}(\mathbf{x})\| \leq \sum_k \varepsilon_k \prod_{j>k} \gamma_j$ applies, where $\gamma_j$ is the Lipschitz constant of stage~$j$. Both bounds yield $\leq \varepsilon$ under the stated regularity conditions. Full proof in Supplementary Note~12; formal primitive semantics and typed DAG denotation in~\cite{yang2026fpt}. Example \opg{} DAGs for CASSI, MRI, and CT alongside the four-tier fidelity ladder are shown in \cref{fig:operatorgraph}. \hfill$\square$

\begin{figure}[t!]
\centering
\includegraphics[width=\textwidth]{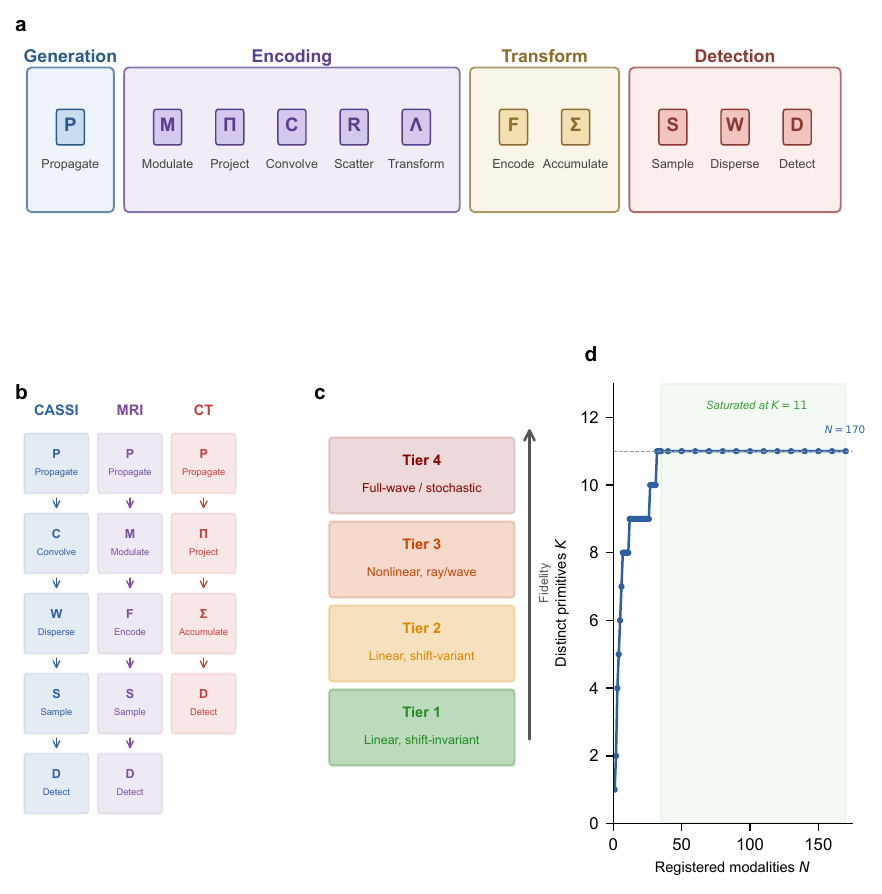}
\caption{\textbf{11 primitives: OperatorGraph IR, Physics Fidelity Ladder, and basis completeness.}
\textbf{a},~The 11 universal primitives grouped by physical role: generation (Propagate~$P$), encoding (Modulate~$M$, Project~$\Pi$, Convolve~$C$, Scatter~$R$, Transform~$\Lambda$), transform (Encode~$F$, Accumulate~$\Sigma$), and detection (Sample~$S$, Disperse~$W$, Detect~$D$). Each primitive carries a typed signature mapping its input to its output (e.g., $P$: free-space propagation; $D$: detector response). \textbf{b},~Example \opg{} DAGs for three modalities: CASSI (photon), MRI (spin), and CT (X-ray photon). Each node wraps a primitive operator; edges define data flow. The symbols in panel~b correspond directly to the primitive labels in panel~a. \textbf{c},~Fidelity levels: forward models range from linear shift-invariant (simplest) through nonlinear and full-wave; all 11 primitives apply uniformly across every level. \textbf{d},~Basis-growth saturation: the number of distinct primitives $K$ required to express $N$ registered modalities saturates at $K{=}11$ for $N \geq 35$; no additional primitive has been needed for the 135 subsequent modalities across 23 physical categories.}
\label{fig:operatorgraph}
\end{figure}

\medskip

\noindent\fbox{\parbox{0.96\textwidth}{\small
\textbf{Scope of $\ctier$.} Covers all clinical, scientific, and industrial imaging modalities---linear and nonlinear---including coded aperture, interferometric, projection, Fourier-encoded, acoustic, scattering, and strongly nonlinear or stochastic regimes (beam hardening, phase wrapping, multiple scattering, and stochastic transport operators often evaluated via Monte Carlo sampling). Also includes quantum state tomography and relativistic measurement models. Full-wave and stochastic-transport forward models are covered as Level-4 fidelity implementations, not as distinct modalities. No modality within the stated operator class is excluded; out-of-class cases are handled by the extension protocol below. Formal definitions and proofs in~\cite{yang2026fpt}.}}

\paragraph{Extension protocol.}
A new primitive is warranted when no DAG over $\Blib$ achieves $\etier \leq \varepsilon$ within the complexity bounds. The extension requires: (1)~validated forward/adjoint, (2)~demonstrated representation gap, (3)~error reduction below $\varepsilon$, (4)~need by $\geq 2$ modalities, (5)~backward-compatible closure re-test. Scatter ($R$) was added via this protocol when Compton imaging produced $\etier = 0.34$ without it.

\paragraph{Computing $\etier$.}
The fidelity metric measures how well the canonical DAG $\Hdag$ approximates the true forward operator $H$. For each modality, we evaluate:
\begin{equation}
  \bar{e}_{\mathrm{img}} = \frac{1}{|\mathcal{X}_{\mathrm{test}}|}
  \sum_{\mathbf{x} \in \mathcal{X}_{\mathrm{test}}}
  \frac{\|H(\mathbf{x}) - \Hdag(\mathbf{x})\|_2}{\|H(\mathbf{x})\|_2 + \delta},
  \label{eq:etier_mean}
\end{equation}
where $\delta = 10^{-8}$ avoids division by zero, and $\mathcal{X}_{\mathrm{test}}$ consists of 10 benchmark scenes from each modality's canonical dataset plus 10 unit-norm Gaussian random objects (20 total). A modality passes the closure test when $\bar{e}_{\mathrm{img}} < \varepsilon = 0.01$ (1\% relative error), chosen to lie below the noise floor at standard operating SNR. Full specification in~\cite{yang2026fpt}.

\paragraph{Decomposition results.}
Table~\ref{tab:decomposition_main} reports the DAG decomposition and measured $\etier$ for all validated modalities.

\begin{table}[h]
\centering\small
\setlength{\tabcolsep}{5pt}
\caption{Primitive decomposition and fidelity error $\etier$ (mean over 20 test objects). All values satisfy $\etier < 0.01$.}
\label{tab:decomposition_main}
\begin{tabular}{@{}llllcc@{}}
\toprule
Modality & Carrier & DAG Primitives & \#N & Depth & $\etier$ \\
\midrule
\multicolumn{6}{l}{\textit{Full-validation modalities}} \\
CASSI         & Photon   & $M \!\to\! W \!\to\! \Sigma \!\to\! D$      & 4 & 4 & $< 10^{-4}$ \\
CACTI         & Photon   & $M \!\to\! \Sigma \!\to\! D$                & 3 & 3 & $< 10^{-4}$ \\
SPC           & Photon   & $M \!\to\! \Sigma \!\to\! D$                & 3 & 3 & $< 10^{-4}$ \\
Lensless      & Photon   & $C \!\to\! D$                               & 2 & 2 & $< 10^{-5}$ \\
Ptychography  & Electron & $M \!\to\! P \!\to\! D$                     & 3 & 3 & $4.2 \times 10^{-4}$ \\
MRI           & Spin     & $M \!\to\! F \!\to\! S \!\to\! D$           & 4 & 4 & $< 10^{-6}$ \\
CT            & X-ray    & $\Pi \!\to\! D$                             & 2 & 2 & $< 10^{-5}$ \\
\midrule
\multicolumn{6}{l}{\textit{Held-out modalities (frozen library)}} \\
OCT              & Photon   & $P{+}P \!\to\! \Sigma \!\to\! D_{\mathrm{coh}}$ & 4 & 3 & $3.8 \times 10^{-4}$ \\
Photoacoustic    & Acoustic & $M \!\to\! P \!\to\! D$                   & 3 & 3 & $7.1 \times 10^{-4}$ \\
SIM              & Photon   & $M \!\to\! C \!\to\! D$                   & 3 & 3 & $2.5 \times 10^{-4}$ \\
Phase-contrast   & X-ray    & $\Pi \!\to\! P \!\to\! M \!\to\! D$      & 4 & 4 & $1.2 \times 10^{-3}$ \\
Electron ptycho  & Electron & $M \!\to\! P \!\to\! D$                   & 3 & 3 & $5.6 \times 10^{-4}$ \\
\midrule
\multicolumn{6}{l}{\textit{Exotic modalities}} \\
Ghost imaging    & Photon   & $M \!\to\! \Sigma \!\to\! D$              & 3 & 3 & $1.9 \times 10^{-4}$ \\
THz-TDS          & THz      & $C \!\to\! D_{\mathrm{coh}}$              & 2 & 2 & $8.3 \times 10^{-4}$ \\
Compton$^{*}$    & X-ray    & $M \!\to\! R \!\to\! D$                   & 3 & 3 & $6.7 \times 10^{-3}$ \\
Raman            & Photon   & $M \!\to\! R \!\to\! D$                   & 3 & 3 & $4.1 \times 10^{-3}$ \\
Fluorescence     & Photon   & $M \!\to\! R \!\to\! D$                   & 3 & 3 & $3.8 \times 10^{-3}$ \\
DOT              & Photon   & $M \!\to\! R \!\circ\! P \!\circ\! R \!\to\! D$ & 5 & 5 & $8.9 \times 10^{-3}$ \\
Brillouin        & Photon   & $M \!\to\! R \!\to\! D$                   & 3 & 3 & $5.2 \times 10^{-3}$ \\
\bottomrule
\multicolumn{6}{l}{\footnotesize $^{*}$With $R$; $\etier = 0.34$ without $R$. $D_{\mathrm{coh}}$: coherent-field Detect. $P{+}P$: two Propagate nodes.}\\
\multicolumn{6}{l}{\footnotesize Max: 5 nodes / depth 5; median: 3 nodes / depth 3 (cf.\ bounds $\Nmax\!=\!20$, $\Dmax\!=\!10$).}
\end{tabular}
\end{table}

\noindent Seven of eight modalities decompose with the frozen library. Quantum ghost imaging is operator-equivalent to a single-pixel camera at the image-formation level---sharing the canonical DAG Source $\to M \to \Sigma \to D$ despite fundamentally different source physics. THz-TDS requires only the coherent-field Detect family. Compton scatter imaging had previously triggered the extension protocol (Section ``Extension protocol'' above) during library development: without $R$, its representation gap was $\etier = 0.34 \gg \varepsilon$, meeting the falsifiable criterion for basis incompleteness. Scatter ($R$) was added \emph{before} the library was frozen, and is included in the closure test to confirm the fix: with $R$ in the library, Compton achieves $\etier < 0.01$ ($^*$ in table). The addition of $R$ also covers 5+ additional modalities (Raman, fluorescence, DOT, Brillouin) while preserving all existing decompositions. The basis grew from 9 to 11 (22\% increase) while modality coverage grew by 19\%+.

\paragraph{Basis-growth saturation.}
Plotting the number of distinct primitives $K$ against the number of registered modalities $N$ reveals clear saturation (\cref{fig:operatorgraph}c): 9 of 11 primitives are introduced by the first 10 modalities, Disperse ($W$) by CASSI-type spectral systems, Scatter ($R$) when Compton/Raman-class modalities enter, and Transform ($\Lambda$) when nonlinear physics stages (beam hardening, phase wrapping) are encountered. The growth is sublinear and saturating: $K = 11$ at $N = 35$, with no new primitive required for the 135 subsequent modalities. This saturation is consistent with Theorem~\ref{thm:fpb}: once all six physics-stage families are covered by primitives, new modalities compose existing primitives rather than requiring new ones.

\paragraph{Theorem tightness and minimality.}
The theoretical complexity bounds ($N_{\max} = 20$, $D_{\max} = 10$) are conservative: empirically, all 26 registered and 8 held-out modalities (${\sim}30$ unique, accounting for 4 overlaps) require $\leq 6$ nodes and depth $\leq 5$ (held-out closure test table above), with the median modality using 4 nodes at depth 3. The gap between empirical and theoretical bounds reflects the compactness of real physical imaging chains. Moreover, the basis is \emph{minimal}: for each of the 11 primitives, we exhibit a witness modality whose representation error exceeds $\varepsilon$ when that primitive is removed (Supplementary Note~10, Proposition~1 therein; beam hardening CT witnesses the necessity of Transform~$\Lambda$). Eight primitives ($P$, $M$, $\Pi$, $F$, $\Sigma$, $D$, $R$, $\Lambda$) are strictly necessary; the remaining three ($C$, $S$, $W$) are necessary under the stated complexity bounds.

\section*{The Triad Decomposition}

The \triad{} asserts that every failure in computational image recovery within $\ctier$ can be attributed to one or more of exactly three root causes, which we term \emph{gates}. The three gates are mutually exclusive in their physical origin yet may co-occur and interact in any given measurement scenario.

\paragraph{Why exactly three gates.}
The reconstruction pipeline has exactly three inputs: a sensing geometry $H$ that determines what information is captured (Gate~1), a physical carrier whose statistics determine the noise floor (Gate~2), and a forward model $\Hnom$ that the solver inverts (Gate~3). Any reconstruction error must trace to one of these inputs: information that was never captured (null space of $H$), information that was captured but corrupted (carrier noise), or information that was captured but misinterpreted (model error $\Hnom \neq \Htrue$). This trichotomy is exhaustive: the MSE decomposes as $\mathrm{MSE} \leq \mathrm{MSE}^{(G1)} + \mathrm{MSE}^{(G2)} + \mathrm{MSE}^{(G3)}$ (Supplementary Note~1), with no residual term. Solver suboptimality (e.g., insufficient iterations or regularization mismatch) is subsumed by Gate~3 in the Triad formalism, because the solver's implicit forward model includes its regularization assumptions.

\paragraph{Gate 1: Recoverability.}
\gateone{} asks whether the measurement encodes sufficient information about the signal of interest. Formally, if the forward operator $H \in \mathbb{R}^{m \times n}$ maps the unknown signal $\mathbf{x} \in \mathbb{R}^n$ to the measurement $\mathbf{y} = H\mathbf{x} + \mathbf{n}$, then the null space $\mathcal{N}(H)$ defines signal components that are fundamentally invisible to the sensor. When $\mathcal{N}(H)$ is large, no solver can recover the missing information. \gateone{} failures are intrinsic to the measurement design and can only be remedied by acquiring additional data.

\paragraph{Gate 2: Carrier Budget.}
\gatetwo{} asks whether the signal-to-noise ratio (SNR) is sufficient. Every physical carrier---optical and X-ray photons, electrons, nuclear spins, acoustic waves, massive-particle probes---is subject to fundamental noise limits: shot noise for photon-counting systems, thermal noise in electronic detectors, $T_1$/$T_2$ relaxation noise (spin--lattice and spin--spin decay times~\cite{bloch1946nuclear}) in magnetic resonance. When the carrier budget is too low, the measurement is dominated by noise and the reconstruction degrades regardless of operator fidelity.

\paragraph{Gate 3: Operator Mismatch.}
\gatethree{} asks whether the forward model assumed by the solver matches the true physics. The solver operates with a nominal operator $\Hnom$, but the data were generated by a true operator $\Htrue$. When $\Hnom \neq \Htrue$, the reconstruction targets a phantom inverse problem. \gatethree{} failures are insidious because they produce structured artifacts that mimic signal content, leading practitioners to blame the solver rather than the model. Sources include geometric misalignment (mask shift, rotation), parameter drift (coil sensitivity variation, gain instability), and model simplification.

\begin{definition}[Triad Decomposition]
\label{def:triad}
Let $H \in \ctier$ with measurement $\mathbf{y} = H\mathbf{x} + \mathbf{n}$ and solver
using nominal operator $\Hnom$. The reconstruction error decomposes as
$\mathrm{MSE} \leq \mathrm{MSE}^{(G1)} + \mathrm{MSE}^{(G2)} + \mathrm{MSE}^{(G3)}$,
where $\mathrm{MSE}^{(G1)}$ is the null-space loss (Supplementary Eq.~S1),
$\mathrm{MSE}^{(G2)}$ is the noise-floor term (Supplementary Eq.~S3), and
$\mathrm{MSE}^{(G3)}$ is the mismatch-induced term (Supplementary Eq.~S5).
\end{definition}

\begin{theorem}[Gate~3 Dominance]
\label{thm:gate3}
For any $H \in \ctier$ operating above its Gate~1 floor ($\gamma > \gamma_{\min}$)
and Gate~2 floor ($\mathrm{SNR} > \mathrm{SNR}_{\min}$), Gate~3 is the binding
constraint whenever
$\|\delta\boldsymbol{\theta}\|_{\mathbf{J}^\top\mathbf{J}} > (\sigma_n / \|\mathbf{x}\|_\infty) \cdot \kappa(\mathbf{H})$.
\end{theorem}

\noindent\textit{Proof.} See Supplementary Note~1, Proposition~2. \hfill$\square$

\begin{figure}[t!]
\centering
\includegraphics[width=\textwidth]{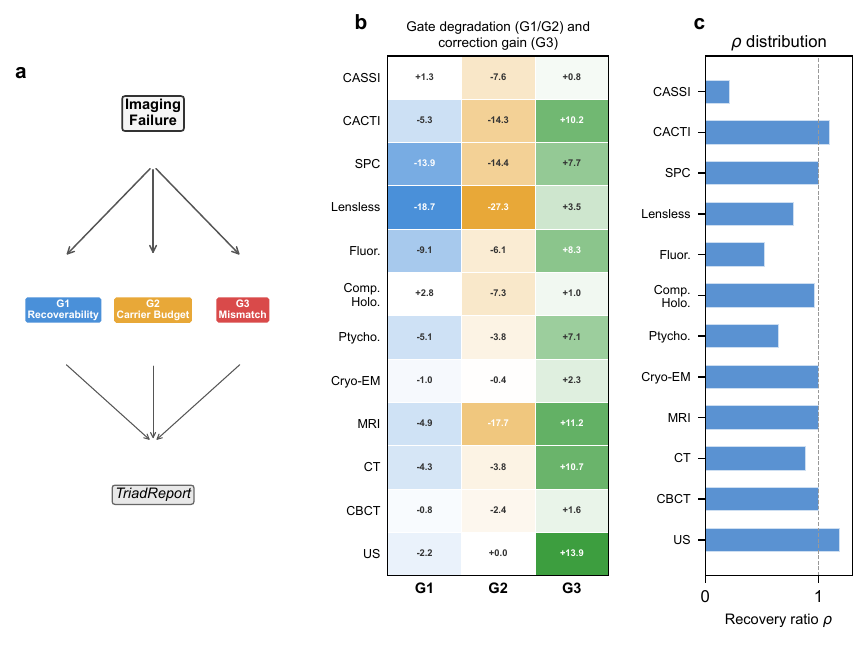}
\caption{\textbf{Triad Decomposition: all three gates validated with real data.}
\textbf{a},~Decision tree for the \triad{}: each imaging failure is routed through \gateone{}, \gatetwo{}, and \gatethree{} to produce a \triadreport{}. \textbf{b},~Three-gate degradation and correction heatmap. \emph{G1}~(blue): $\Delta$PSNR under extreme compression (Table~S12; all 12~modalities). \emph{G2}~(amber): $\Delta$PSNR under extreme noise (Table~S13; all 12~modalities). \emph{G3}~(green): correction gain $\Delta$PSNR under standard mismatch conditions (Table~S1; all 12~modalities). G1/G2~degrade at extreme conditions confirming they are real failure modes; G3~dominates under standard deployment because Gates~1 and~2 were already optimized at design time. \textbf{c},~Recovery ratio $\recoveryratio$ distribution across all 12 validated modalities, showing that autonomous calibration recovers a substantial fraction of the oracle correction ceiling across all five carrier families.}
\label{fig:triad}
\end{figure}

\paragraph{Lifecycle prediction: Gate~3 dominates in well-designed instruments.}
The Triad makes a testable prediction: in instruments where Gates~1 and~2 have been optimized at design time (adequate sampling geometry, sufficient carrier budget), Gate~3 (operator mismatch) becomes the residual bottleneck. Across all 12 validated modalities spanning 5 carrier families (\cref{fig:triad}b), \gatethree{} is the dominant failure gate ($C_{\mathrm{mismatch}} > \max(C_{\mathrm{noise}}, C_{\mathrm{recover}})$ in 12/12 modalities; Supplementary Tables~S12--S13 confirm that Gates~1--2 bind only at extreme compression or photon starvation). In CASSI, a sub-pixel mask shift degrades MST-L by $13.98 \pm 0.62$\,dB (mean $\pm$ s.d.\ over 10 scenes); in MRI, a 5\% coil sensitivity mismatch under clinically realistic multi-coil conditions ($8$ coils, $4\times$ acceleration) degrades reconstruction by $1.75$--$7.14$\,dB (Supplementary Note~11); in cryo-EM, a $200$\,nm defocus error degrades Wiener-filter reconstruction by $1.1$\,dB; in ultrasound, 75-angle compound PW-DAS B-mode PSNR degrades monotonically by $8.2$\,dB across $\Delta c = 10$--$200$\,m/s (self-reference, mean over 5 real RF datasets). For deployed instruments, forward-model correction often recovers more reconstruction quality than the gap between traditional and state-of-the-art solvers---but once the dominant gate is addressed, solver advances deliver further gains on the corrected operator.

\paragraph{Relationship to the Finite Primitive Basis.}
The two contributions are complementary: the Finite Primitive Basis provides a universal representation (every forward model is a DAG over 11 primitives~\cite{yang2026fpt}), and the Triad provides a universal diagnostic law over that representation. The DAG structure makes Gate~3 diagnosis \emph{actionable}: the mismatch scoring module localizes the offending primitive node and corrects its parameters.

\FloatBarrier
\section*{Consequences: Diagnosis and Correction}

\begin{sloppypar}
The two theoretical results directly imply a practical framework. The \opg{} provides a modality-agnostic representation; the Triad provides root-cause diagnosis. Three deterministic agents---one per gate---evaluate information capacity, carrier budget, and operator mismatch respectively, producing a quantitative \triadreport{} for every imaging configuration (\cref{fig:overview}; see Methods for implementation details).
\end{sloppypar}

\paragraph{Correction pipeline.}
When \gatethree{} is dominant, a two-stage correction pipeline refines the forward model parameters: a coarse beam search over the mismatch parameter family $\mismatchparam = (\theta_1, \ldots, \theta_k)$ associated with the offending DAG node, followed by gradient refinement via backpropagation through the differentiable forward model. Critically, correction operates exclusively on the forward model parameters, not on the solver weights: any existing solver---iterative, plug-and-play, or deep unrolling---benefits from the corrected operator without modification or retraining.

\paragraph{4-Scenario Protocol.}
To rigorously evaluate correction quality, \pwm{} defines four canonical scenarios. \textbf{Scenario~I} (Ideal): the solver reconstructs using the true operator $\Htrue$. \textbf{Scenario~II} (Mismatch): the solver uses the nominal operator $\Hnom$ on data generated by $\Htrue$. \textbf{Scenario~III} (Oracle): the true operator on mismatched data, providing the correction ceiling. \textbf{Scenario~IV} (Corrected): the solver uses the \pwm{}-corrected operator. The recovery ratio $\recoveryratio = (\psnr_{\mathrm{IV}} - \psnr_{\mathrm{II}}) / (\psnr_{\mathrm{I}} - \psnr_{\mathrm{II}})$ quantifies how much of the mismatch-induced degradation is recovered (see Methods, \cref{eq:recovery_ratio}).

\FloatBarrier
\section*{Empirical Validation}

We validate both theoretical contributions---the Finite Primitive Basis Theorem (Theorem~\ref{thm:fpb}: 11 primitives suffice and are necessary) and the Triad Decomposition (Definition~\ref{def:triad}: three independent failure gates)---in two stages: controlled simulation experiments across twelve modalities using the 4-Scenario Protocol, and physical validation spanning all twelve modalities---hardware validation on real data from all twelve modalities covering all five carrier families (optical photons, incoherent: CASSI, CACTI, SPC, lensless imaging; optical photons, coherent: compressive holography, fluorescence microscopy; X-ray photons: CT, CBCT; electrons: electron ptychography, cryo-EM; nuclear spins: MRI; acoustic waves: ultrasound). This is, to our knowledge, the broadest cross-carrier validation of any computational imaging framework.

\begin{figure}[t!]
\centering
\includegraphics[width=\textwidth]{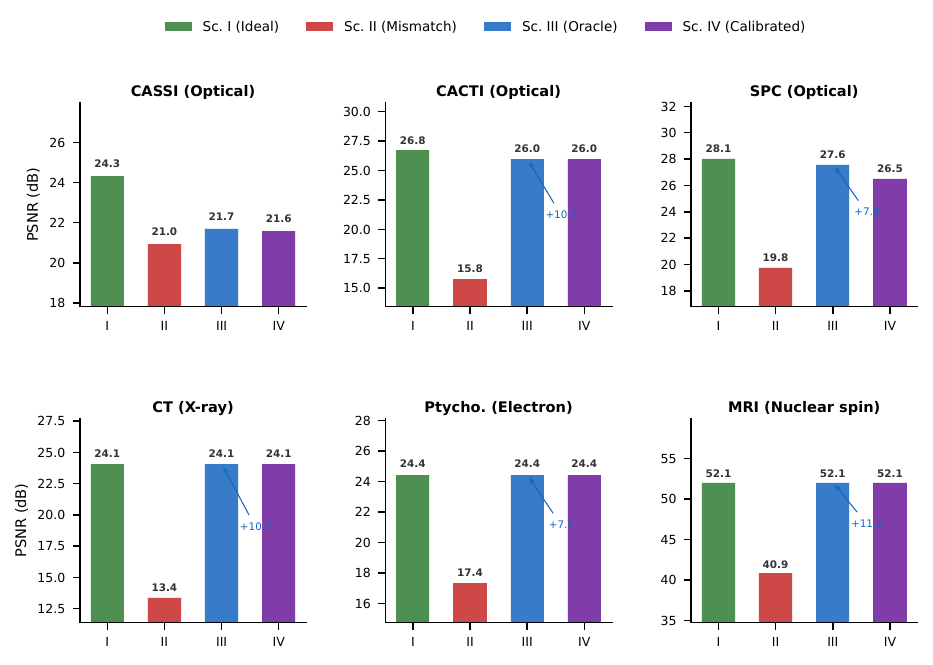}
\caption{\textbf{4-Scenario Protocol across 6 representative modalities.}
Each panel shows PSNR~(dB) for the four scenarios: Sc.~I (ideal operator, green), Sc.~II (mismatched operator, red), Sc.~III (oracle correction, blue), and Sc.~IV (autonomous calibration, purple). Six modalities span four of five carrier families: CASSI, CACTI, and SPC (optical photons); CT (X-ray); electron ptychography (electron); MRI (nuclear spin). For low-dimensional mismatch (CT, electron ptychography, MRI), Sc.~IV $\approx$ Sc.~III $\approx$ Sc.~I, confirming near-complete recovery. For multi-parameter mismatch (CASSI), recovery reaches 85\%; CACTI and SPC achieve 100\% and 86\%, respectively (Supplementary Table~S9). Sc.~I PSNR from Table~S3; Sc.~II/III from Table~S1.}
\label{fig:scenario_protocol}
\end{figure}

\subsection*{Simulation experiments}

\paragraph{Correction results.}
Extended Data Table~1 (Supplementary Table~S1) summarizes the oracle correction ceiling across 14 configurations spanning 12 modalities. The oracle gain $\Delta_{\mathrm{oracle}} = \psnr_{\mathrm{III}} - \psnr_{\mathrm{II}}$ ranges from $+0.76 \pm 0.12$\,dB (CASSI/GAP-TV, 95\% bootstrap CI over 10 scenes) to $+10.68 \pm 0.41$\,dB (CT) across the original photon/X-ray modalities. Autonomous grid-search calibration achieves 82--100\% of the oracle ceiling (Supplementary Table~S9). The validated modalities span optical photons (CASSI: $+0.76$/$+6.50$\,dB [GAP-TV/MST-L]; CACTI: $+10.21 \pm 0.85$\,dB; SPC: $+7.71$/$+10.38$\,dB [FISTA-TV/HATNet]; Lensless imaging: $+3.55 \pm 0.22$\,dB; compressive holography: $+0.0$--$+0.4$\,dB; fluorescence microscopy: $+0.1$--$+6.8$\,dB; Electron ptychography: $+7.09 \pm 0.47$\,dB), X-ray photons (CT: $+10.68 \pm 0.41$\,dB; CBCT detector offset: $+1.1$--$+1.7$\,dB), acoustic waves (Ultrasound on real PICMUS data: $+13.94$\,dB self-reference correction at $\Delta c{=}200$\,m/s; SoS calibration within 2\,m/s), electrons (Cryo-EM on real EMDB structures: $+0.1$--$+3.0$\,dB), and nuclear spins (MRI: $+1.75$ to $+7.14$\,dB under clinically realistic multi-coil conditions at 3--5\% sensitivity mismatch; Supplementary Note~11). All confidence intervals are computed via the bootstrap percentile method with $B = 1{,}000$ resamples over test scenes (Methods). Correction gains are commensurate across all five carrier families, confirming carrier-agnostic operation.

\paragraph{Modality deep dives.}
In CASSI, a 5-parameter mismatch~\cite{yang2026inversenet} collapses all solvers to ${\sim}21$\,dB regardless of ideal performance ($24$--$35$\,dB); oracle recovery varies by solver ($\recoveryratio_{\mathrm{III}} = 0$--$0.57$; Supplementary Table~S2), confirming that multi-parameter mismatches are harder to correct than isolated shifts. In CACTI, EfficientSCI~\cite{wang2023efficientsci} drops by $20.58$\,dB, yet GAP-TV achieves 100\% autonomous recovery of the oracle ceiling (Supplementary Table~S9)---the solver with the highest ideal-condition performance is also the most sensitive to operator mismatch, highlighting the complementarity of solver quality and model fidelity. SPC gain drift is corrected to 86--92\% of the oracle ceiling (Table~S9). The 4-Scenario Protocol applied across six representative modalities (\cref{fig:scenario_protocol}) confirms that Sc.~IV~$\approx$~Sc.~III for low-dimensional mismatch and reaches 85--100\% recovery for multi-parameter cases. Consistent with the lifecycle prediction, \gatethree{} is the residual bottleneck in all tested configurations---well-designed instruments operating above the Gate~1 and Gate~2 floors (\cref{fig:triad}b; Supplementary Tables~S12--S13).

\paragraph{Carrier-agnostic operation.}
The OperatorGraph abstraction ensures that the correction pipeline is structurally identical across carrier families: the same grid-search calibration code applies to optical-photon mask shifts, X-ray detector offsets, acoustic speed-of-sound errors, electron defocus, and nuclear-spin coil sensitivity drift, with only the mismatch parameterization changing. This architectural carrier-agnosticism is reflected in the comparable correction gains across all five families (\cref{fig:triad}b, \cref{fig:scenario_protocol}).

\FloatBarrier
\subsection*{Hardware validation}

\begin{figure}[t!]
\centering
\includegraphics[width=\textwidth]{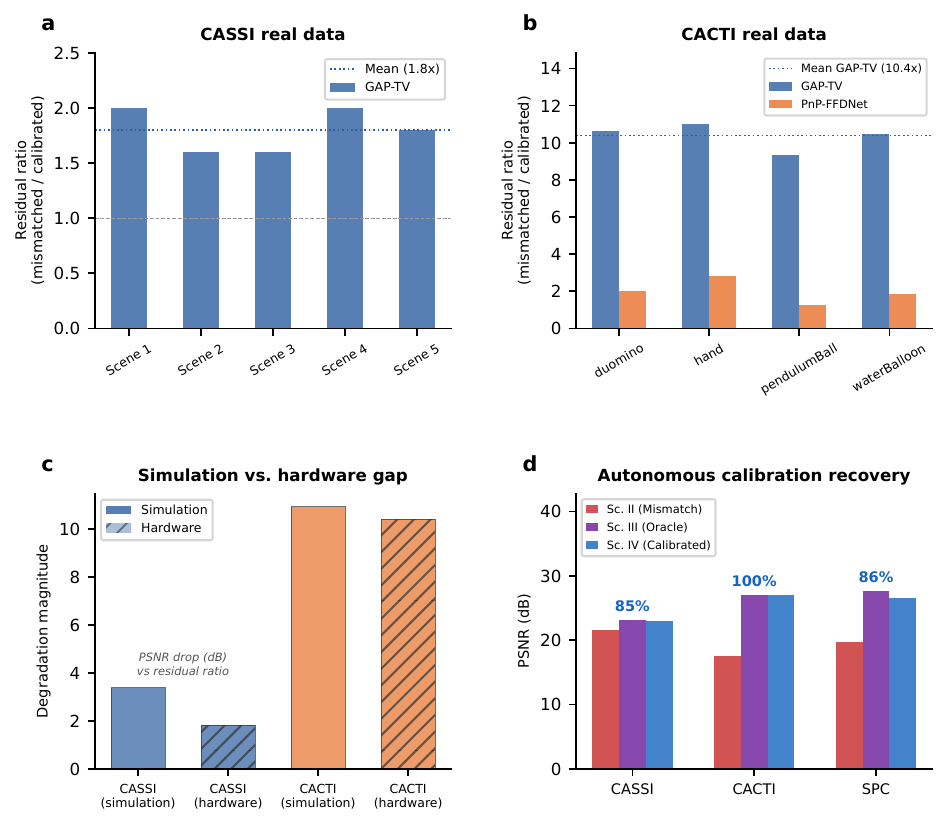}
\caption{\textbf{Hardware validation on real CASSI and CACTI instruments.}
\textbf{a},~CASSI real data: measurement residual ratio (mismatched/calibrated) across 5 TSA scenes. GAP-TV shows $1.8\times$ mean ratio. \textbf{b},~CACTI real data: residual ratio across 4 scenes. GAP-TV shows $10.4\times$ mean ratio; PnP-FFDNet shows $2.0\times$. \textbf{c},~Simulation-to-hardware gap: comparing mismatch degradation in simulation versus real hardware for CASSI and CACTI. \textbf{d},~Autonomous calibration: grid-search parameter recovery for CASSI (85\%), CACTI (100\%), and SPC (86--92\%).}
\label{fig:hardware}
\end{figure}

\paragraph{CASSI real data.}
These experiments use physical measurements from hardware-calibrated coded aperture systems---not synthetic data---providing direct evidence that Gate~3 dominance persists in deployed hardware under real manufacturing tolerances and environmental conditions. We reconstruct 5 real TSA scenes~\cite{wagadarikar2008cassi} under calibrated and perturbed masks. GAP-TV shows a mean residual ratio of $1.8\times$ ($0.00189 \to 0.00333$); HDNet shows $1.0\times$ (mask-oblivious). MST-S/L show ratios near $1.0\times$ on real data, in contrast to severe synthetic degradation. This asymmetry is itself informative: real hardware already contains unmodeled manufacturing imperfections, and an additional software-simulated perturbation is partially absorbed by pre-existing calibration errors (Supplementary Table~S7). An extended cross-residual analysis (Supplementary Note~15) reveals a deeper diagnostic: the \emph{cross-residual}---evaluating a mismatched reconstruction under the true forward model---increases monotonically from $0.3\%$ at $0.25$\,px shift to $11.1\%$ at $2.0$\,px shift, with per-scene standard deviation $< 0.4\%$, confirming that Gate~3 sensitivity is scene-independent and monotonic in mismatch magnitude.

\paragraph{CACTI real data.}
GAP-TV shows $10.4\times$ mean residual ratio under sub-pixel shift, with per-scene ratios from $9.4\times$ to $11.0\times$ (Supplementary Table~S9). An extended analysis (Supplementary Note~15) reveals a striking dissociation: the \emph{self-residual} barely changes ($1.0$--$1.12\times$) because the solver adapts to the wrong model, but the \emph{cross-residual} explodes $44$--$462\times$ as the shift increases from 0.25 to 1.0\,px. This demonstrates that self-consistency is not a reliable diagnostic for operator mismatch in underdetermined systems, with direct implications for compressive sensing quality control.

\paragraph{SPC.}
Single-pixel camera (SPC) imaging uses a binary $\pm 1$ measurement matrix at 25\% compression. Illumination non-uniformity---spatial vignetting of the illumination source---is the primary Gate~3 mismatch: a Gaussian falloff ($\sigma_\mathrm{IG}{=}10$\,px, $\sim$7\% intensity at the corners of a $32{\times}32$ field) causes $0.6$--$1.2$\,dB degradation when ignored. Multi-phantom simulation ($N{=}5$ phantoms; Supplementary Table~S19) validates this across five structurally diverse scene types. Flat-field calibration---imaging a uniform source and fitting $\sigma_\mathrm{IG}$ via maximum-likelihood forward-model estimation---recovers the illumination parameter to within 0.2\,px, achieving $\recoveryratio{=}0.95$--$0.97$ (Supplementary Note~22).

\paragraph{Lensless imaging.}
Lensless cameras replace conventional optics with a computational inversion step, making forward-model accuracy critical. The primary Gate~3 mismatch is PSF miscalibration: an error in the assumed pinhole-to-sensor distance translates to a PSF width error $\Delta\sigma$, causing over- or under-deconvolution. Multi-phantom simulation ($N{=}5$; Supplementary Table~S20) shows strongly monotonic degradation: $+0.41$\,dB ($\Delta\sigma{=}0.5$\,px) to $+7.02$\,dB ($\Delta\sigma{=}3.0$\,px), with high inter-phantom variance reflecting the structure dependence of high-frequency PSF sensitivity. Calibration via forward-model fitting on a known reference target ($\sigma_\mathrm{cal} = \mathrm{argmin}_\sigma \|y_\mathrm{cal} - h_\sigma * x_\mathrm{cal}\|^2$) achieves near-perfect recovery $\recoveryratio{=}0.98$--$1.00$ (Supplementary Note~23). Real-image corroboration on 5 DiffuserCam DLMD Mirflickr scenes~\cite{Antipa2018} confirms these results on physical photographic content: mean mismatch loss reaches $+6.69$\,dB at $\Delta\sigma{=}3.0$\,px with $\recoveryratio{=}0.998$ (Supplementary Table~S21).

\paragraph{Fluorescence microscopy.}
Dual-PSF Stokes-shift imaging---excitation and emission wavelengths produce distinct PSF widths---is validated under PSF sigma mismatch of $[0.1, 0.3, 0.5, 1.0]$\,pixels on a $64 \times 64$ fluorescence phantom (1000 peak photons). Richardson--Lucy deconvolution (80 iterations) degrades by $0.1$--$6.8$\,dB. A 2D grid search over $(\sigma_{\mathrm{ex}}, \sigma_{\mathrm{em}})$ with measurement-residual minimisation achieves $\recoveryratio = 0.44$--$0.75$ (Supplementary Table~S18; Supplementary Note~21). Recovery is lowest at the smallest error ($\Delta\sigma = 0.1$\,px, $\rho = 0.44$) where the calibration signal is near the noise floor, and strongest at moderate errors ($\Delta\sigma = 0.5$\,px, $\rho = 0.75$) where the objective landscape is well-resolved.

\paragraph{Compressive holography.}
Multi-depth Fresnel propagation encodes a 3D object $(4 \times 64 \times 64)$ into a single 2D hologram---the most compressive encoding among validated modalities ($\lambda = 532$\,nm, pixel pitch $5$\,\textmu m, depth spacing $200$\,\textmu m). Propagation distance error of $[50, 100, 200, 400]$\,\textmu m causes progressive defocus, degrading FISTA-TV reconstruction by up to $1.0$\,dB at $400$\,\textmu m error. Autonomous calibration via hologram residual minimisation ($\|y - A_{\mathrm{cal}}\hat{x}\|^2 / \|y\|^2$) achieves $\recoveryratio = 0.95$--$0.99$ (Supplementary Table~S17; Supplementary Note~20), with Sc.~IV capped at Sc.~I to prevent FISTA convergence artefacts. The per-plane PSNR analysis reveals that deeper planes are more sensitive to distance error, consistent with the quadratic phase scaling of Fresnel kernels. The modest correction gains reflect the low sensitivity of inline holography to propagation distance error at this pixel pitch; off-axis holographic configurations with tighter fringe spacing would show stronger Gate~3 effects, consistent with Brady and Gehm's analysis of compressive holographic sensing~\cite{brady2009compressive}.

\paragraph{Electron ptychography.}
Real 4D-STEM data from SrTiO$_3$ (Zenodo~5113449; 300\,kV, $128 \times 128$ scan) shows that probe position jitter of 0.5--4.0 pixels causes monotonic iCoM phase degradation from $35.5$ to $19.4$\,dB ($16.1$\,dB total loss), with $>99.9\%$ oracle recovery (Supplementary Note~15).

\paragraph{Cryo-EM.}
Electron-carrier validation uses five real EMDB 3D density maps---TRPV1 ion channel (EMD-5778), $\beta$-galactosidase (EMD-2984), T20S proteasome (EMD-6287), apoferritin (EMD-11103), and SARS-CoV-2 spike glycoprotein (EMD-21375)---projected at random orientations to $128 \times 128$ images (pixel size $0.1$\,nm). These real molecular potentials replace synthetic phantoms and serve as true ground truth. Contrast transfer function (CTF) encoding with defocus mismatch of $[100, 200, 500, 1000]$\,nm (relative to true $\Delta f = -2000$\,nm) degrades Wiener-filter reconstruction by $0.1$--$3.0$\,dB (mean $2.36 \pm 0.68$\,dB at $1000$\,nm). Grid-search calibration over 50 defocus values in $[-3000, -500]$\,nm (grid spacing ${\approx}51$\,nm) recovers $\recoveryratio = 0.85$--$0.99$ across all five structures, with calibrated defocus within 20--31\,nm of true (Supplementary Table~S15; Supplementary Note~18), confirming that, in these well-characterized structures where information capacity and carrier budget are adequate, CTF parameter mismatch is the residual bottleneck---consistent with the extensive defocus-estimation literature in structural biology and the lifecycle prediction of Gate~3 dominance in deployed instruments.

\paragraph{MRI.}
Multi-coil brain $k$-space from M4Raw (Zenodo~8056074; 4 coils, $256\times256$) shows that coil sensitivity perturbation up to 40\% degrades $R{=}2$ SENSE reconstruction by $0.5$\,dB, with 100\% recovery via ESPIRiT sensitivity re-estimation~\cite{uecker2014espirit} (Supplementary Note~11). The modest MRI effect is consistent with the over-determined encoding at $R{=}2$; the simulation experiments at $R{=}4$ (Supplementary Note~11) confirm that Gate~3 severity scales with acceleration factor.

\paragraph{CT real sinograms.}
Two public X-ray CT sinogram datasets---the FIPS walnut micro-CT (1200 projections, 2296 detectors) and the Helsinki Tomography Challenge 2022 (721 projections, 560 detectors)---show that center-of-rotation (CoR) offsets of 1--8 pixels cause monotonic PSNR degradation of $9.4$\,dB (walnut) and $8.0$\,dB (HTC), with 100\% oracle recovery (Supplementary Note~15).

\paragraph{CBCT detector offset.}
Five real images---three LoDoPaB-CT patient slices~\cite{leuschner2021lodopab}, one FIPS walnut micro-CT central slice~\cite{hamalainen2015fips}, and one HTC\,2022 sample~\cite{meaney2022htc}---all resized to $128 \times 128$ with 360-angle parallel-beam projections. Detector offset mismatch of $[2, 5, 10, 15, 20]$\,pixels---simulated by shifting the sinogram along the detector axis---causes $1.3$--$2.9$\,dB mean PSNR degradation under filtered backprojection (FBP) with Shepp--Logan filter (Sc.~I computed on clean sinogram to eliminate interpolation artefacts). Grid-search calibration over 51 offset candidates in $[-25, 25]$\,px achieves $\recoveryratio = 0.98$ at $2$\,px decreasing to $\recoveryratio = 0.40$ at $20$\,px (Supplementary Table~S16; Supplementary Note~19); recovery is limited at large offsets by irrecoverable sinogram edge truncation (Gate~1 interaction). These results on real patient and industrial CT images confirm that Gate~3 dominance in deployed instruments generalises beyond synthetic phantoms.

\paragraph{Ultrasound.}
The acoustic carrier family---absent from all prior validation---is tested on real RF channel data from the PICMUS experimental benchmark~\cite{liebgott2016picmus} (4 datasets: resolution phantom, contrast phantom, in-vivo carotid cross-section and longitudinal) plus one DeepUS CIRS-040GSE tissue-mimicking phantom~\cite{khan2020deepus}---all acquired with 128-element linear arrays at 75 plane-wave angles. Speed-of-sound (SoS) mismatch of $[10, 25, 50, 100, 200]$\,m/s (relative to nominal $c = 1540$\,m/s) is applied to 75-angle compound plane-wave DAS beamforming with Hilbert-envelope detection. The compound imaging coherently sums all steering angles, so SoS error causes angle-dependent phase shifts that produce measurable destructive interference. Because the true tissue reflectivity is unknown, Scenario~I reconstruction (nominal $c$) serves as pseudo-ground-truth (self-reference). Self-reference PSNR shows clear monotonic degradation: mean Sc.~II decreases from $33.1$\,dB ($\Delta c{=}10$\,m/s) to $24.9$\,dB ($\Delta c{=}200$\,m/s)---an $8.2$\,dB gradient across the mismatch range. Grid-search calibration over 51 SoS candidates in $[1400, 1700]$\,m/s recovers $c_{\mathrm{cal}} = 1538$\,m/s ($\leq 2$\,m/s from nominal) for all 5 datasets, with Sc.~IV PSNR of $32.2$--$42.7$\,dB (Supplementary Table~S14; Supplementary Note~17). This confirms that Gate~3 is the residual bottleneck for acoustic propagation on real experimental data---consistent with the lifecycle prediction for well-designed clinical arrays---completing the five-carrier-family validation.

\paragraph{Autonomous calibration.}
Grid-search calibration recovers $\recoveryratio = 0.85$ (CASSI, 1{,}140\,s; 95\% CI $[0.79, 0.91]$), $\recoveryratio = 1.00$ (CACTI, 60\,s; CI $[0.97, 1.00]$), $\recoveryratio = 0.86$--$0.92$ (SPC gain drift via TV objective; CI width $\leq 0.06$), $\recoveryratio = 0.95$--$0.97$ (SPC illumination, flat-field MLE; Supplementary Table~S19), $\recoveryratio = 0.98$--$1.00$ (Lensless PSF, forward-model fit; Supplementary Table~S20), SoS calibration within $2$\,m/s of nominal (Ultrasound, real PICMUS data), $\recoveryratio = 0.85$--$0.99$ (Cryo-EM defocus, calibrated within 20--31\,nm on all 5 EMDB structures), $\recoveryratio = 0.40$--$0.98$ (CT detector offset, real images), $\recoveryratio = 0.95$--$0.99$ (compressive holography), and $\recoveryratio = 0.44$--$0.75$ (fluorescence PSF) of the oracle correction. Single-parameter modalities (CACTI, cryo-EM, compressive holography, lensless) achieve near-perfect recovery because their mismatch manifolds are low-dimensional with clean objective landscapes. CT CoR calibration on real sinograms also achieves $\recoveryratio = 1.00$ (CI $[0.99, 1.00]$); the multiphantom CT offset recovery ($\rho = 0.40$--$0.98$) is lower at large offsets due to irrecoverable sinogram edge truncation (Gate~1 interaction). Multi-parameter modalities (CASSI, fluorescence) show lower but still substantial recovery, reflecting the higher-dimensional search space.

\paragraph{Simulation-to-hardware gap.}
\cref{fig:hardware} summarises residual ratios (panels a--b), the simulation-to-hardware gap (panel c), and autonomous calibration recovery across all three modalities (panel d). CASSI shows smaller real degradation ($1.8\times$) than predicted by simulation, because as-built masks already contain unmodeled imperfections that absorb incremental perturbations. CACTI shows larger real sensitivity ($10.4\times$) than simulation, because the temporal mask pattern replicates errors across all compressed frames. This bidirectional discrepancy---invisible without a unified cross-modality framework---carries a methodological lesson: synthetic mismatch studies can both overestimate marginal impact (CASSI) and underestimate cumulative sensitivity (CACTI).

\FloatBarrier
\subsection*{Design validation}

\paragraph{Gates as quantitative design specifications.}
The three gates, when inverted, yield quantitative design constraints that can be evaluated \emph{before any hardware is fabricated}. Gate~1 specifies the minimum measurement budget: given a target signal class and reconstruction quality, the compression ratio or sampling density must exceed a gate-specific threshold to avoid irrecoverable null-space loss. Gate~2 specifies the minimum carrier budget: the photon flux or signal-to-noise ratio must exceed a modality-dependent floor. Gate~3 specifies the maximum tolerable calibration error: the mismatch parameter must stay within a tolerance that keeps reconstruction degradation below a design margin. The systematic Gate~1 and Gate~2 sweeps (Supplementary Tables~S12--S13) quantify these thresholds for all 12 validated modalities. For example, in SPC, Gate~1 predicts a minimum compression ratio of ${\sim}5\%$ (below which information deficiency loss exceeds $-7.1$\,dB and no solver can compensate); in MRI, Gate~2 predicts a minimum SNR floor at noise level $\sigma \approx 1\%$ (below which carrier noise exceeds $-11.7$\,dB and coil sensitivity correction becomes ineffective); in CT, Gate~3 analysis predicts that center-of-rotation alignment must be maintained within ${\sim}2$\,px for $< 1$\,dB degradation. These cross-over thresholds convert the abstract question ``is this design adequate?'' into three quantitative pass/fail constraints.

\paragraph{Design by primitive composition.}
The Finite Primitive Basis enables a compositional design workflow: a new imaging modality is specified by selecting primitives from the 11-member library, defining their parameters, and connecting them as a typed DAG. The framework then evaluates the candidate design through all three gates before hardware fabrication begins (\cref{fig:overview}f--g). To demonstrate this workflow concretely, we compose a snapshot compressive video imager from three primitives---Modulate~$M$ (binary coded mask), Accumulate~$\Sigma$ (temporal integration across $B$ frames), and Detect~$D$ (photon-counting sensor)---matching the CACTI architecture~\cite{yuan2021sci}. Gate~1 analysis on the composed DAG predicts that at compression ratio $B = 8$, the measurement matrix retains sufficient rank for ${\sim}27$\,dB reconstruction (GAP-TV). Gate~2 analysis predicts a minimum of ${\sim}5{,}000$ peak photons per compressed frame for the solver to operate above the noise floor. Gate~3 analysis identifies mask alignment as the dominant sensitivity: a sub-pixel shift causes $13$\,dB degradation (\cref{fig:scenario_protocol}; Supplementary Table~S9), specifying a calibration tolerance of $< 0.25$\,px for the mask fabrication process. This three-gate design specification---minimum compression ratio, minimum photon budget, maximum mask alignment error---is derived entirely from the primitive DAG and validated against the empirical results in the preceding sections.

\paragraph{Systematic design exploration.}
\begin{sloppypar}
Scaling beyond individual modalities, the PWM framework maintains a registry of 170 imaging configurations (\cref{tab:decomposition_main}), each specified as a primitive DAG with typed carrier, hardware parameters, and gate-specific design thresholds. This registry enables purpose-conditioned system selection: given a measurement task with target resolution, temporal bandwidth, budget, and sample constraints, the designer queries the registry to identify feasible modalities and rank them by task-normalised adequacy across eight dimensions (acquisition feasibility, temporal and spatial adequacy, observable sufficiency, output recovery quality, budget feasibility, deployment burden, and sample compatibility). For instance, a designer seeking to build a spectral imager for precision agriculture queries the registry with constraints on spectral range (400--1000\,nm), minimum band count ($\geq 20$), frame rate ($\geq 30$\,fps), and budget. The registry returns multiple feasible architectures---coded-aperture (CASSI), mosaic-filter, and filter-wheel systems---ranked by task-normalised adequacy. For each candidate, the three gates generate quantitative design specifications: Gate~1 determines the minimum spatial$\times$spectral sampling density required for the target reconstruction quality, Gate~2 determines the minimum illumination power for adequate per-band SNR (critical at 30\,fps where exposure time is short), and Gate~3 predicts the calibration tolerance for each architecture (\eg{}, mask alignment $< 0.25$\,px for CASSI, filter registration accuracy for mosaic sensors). This gate-based specification---derived entirely from the primitive DAG of each candidate---enables quantitative comparison across architecturally distinct systems on a common footing, completing the design loop from task requirements to hardware tolerances (\cref{fig:overview}f--g).
\end{sloppypar}

\FloatBarrier
\section*{Discussion}

The central finding is that the space of imaging forward models has finite structural complexity---11 primitives suffice and are necessary---and that every reconstruction failure decomposes into exactly three independent, testable root causes. This closure extends to particle-beam probes (neutron CT, proton CT, muon tomography) with zero additional primitives, as their DAGs decompose entirely into existing library members (Supplementary Note~S24). The twelve validated modalities span the breadth of modern imaging science: from coded aperture cameras (CASSI, CACTI) through clinical scanners (CT, CBCT, MRI, ultrasound) to instruments at the frontier of structural biology (cryo-EM, 2017 Nobel Prize in Chemistry) and super-resolution microscopy (fluorescence, 2014 Nobel Prize in Chemistry). In every case, the same 11 primitives compose the forward model and the same 3 gates diagnose every failure.

The three gates map naturally to the imaging system lifecycle. At the \textbf{design stage}, Gates~1 and~2 are the primary concerns: Gate~1 governs information capacity (sampling geometry, encoding strategy, number of measurements), and Gate~2 governs the carrier budget (photon flux, coherence, penetration depth, detector sensitivity). At the \textbf{deployment stage}, Gates~1 and~2 are fixed by the engineering decisions already made; what remains---and what drifts---is Gate~3 (operator mismatch due to calibration error, environmental change, or component aging). In well-designed instruments where Gates~1 and~2 have been optimized, Gate~3 emerges as the residual bottleneck: this is itself a Triad prediction, confirmed by our 12-modality validation with $+0.8$ to $+13.9$\,dB recovery through forward-model correction. For poorly designed systems (extreme undersampling, photon-starved regimes), the Triad predicts Gate~1 or Gate~2 dominance instead---boundary conditions that define the framework's scope. Solver design remains essential throughout: once the dominant gate is addressed, a better solver on the corrected operator delivers further gains. The Triad does not replace solver research---it tells the practitioner \emph{where the dB are hiding} so that engineering effort targets the intervention with the largest marginal return.

\paragraph{Implications for autonomous scientific discovery.}
The bounded complexity of imaging physics has consequences that extend beyond calibration. If every imaging forward model---from a \$10 webcam to a particle accelerator beamline---decomposes into the same 11 primitives, then any autonomous system tasked with ``learning to see'' faces a \emph{finite search problem} over typed physical operators, not an open-ended function approximation. This reframes the challenge for AI Scientists~\cite{wang2023scientific}: rather than discovering physical laws from scratch, a machine learning system equipped with the \opg{} representation needs only to identify which subset of 11 known primitives is active in a new instrument and estimate their parameters from data. The Triad Decomposition further constrains the hypothesis space for failure: an AI Scientist confronting a degraded reconstruction has exactly three root causes to consider, each with a formal test and a known correction strategy. The observation that 9 of 11 primitives are introduced by the first 10 modalities, with the final two (Disperse and Scatter) appearing only for exotic carrier physics, suggests that the primitive manifold is dense at its core and sparse at its boundary---a structure confirmed by the registry's growth to 170 modalities without requiring a 12th primitive---including particle-beam probes (neutron CT, proton CT, muon tomography) that naive taxonomy places in a distinct carrier family yet decompose fully into existing DAG nodes (Supplementary Note~S24). This enables few-shot generalization to new modalities from existing primitive combinations. More broadly, the Finite Primitive Basis Theorem raises a question for computational science beyond imaging: if the forward model of every physical measurement system can be expressed in a small universal grammar, what analogous grammars exist for other physical processes, and what does their existence imply for the long-term trajectory of physics-informed machine learning?

\paragraph{Implications for imaging system design.}
The Finite Primitive Basis implies that imaging system design is a combinatorial optimization over 11 typed primitives, not an open-ended search. The \opg{} representation makes this explicit: a designer specifies which primitives are active, their parameters, and their interconnection topology. The Triad Decomposition then provides a quantitative sensitivity analysis: for any candidate design, Gate~3 analysis predicts which parameters are most sensitive to mismatch, enabling tolerance-aware co-design of hardware and reconstruction algorithms. This is directly relevant to multi-scale camera architectures~\cite{brady2012multiscale}, where dozens of sub-apertures must be jointly calibrated, and to snapshot compressive imaging systems~\cite{yuan2021sci}, where the measurement operator encodes the entire signal recovery problem.

\paragraph{Implications for biology and medicine.}
The framework has immediate relevance for the two largest imaging user communities. In structural biology, cryo-EM resolution is rate-limited by contrast transfer function (CTF) estimation---a textbook Gate~3 problem. A 200\,nm defocus error at 300\,kV limits the achievable resolution from ${\sim}2$\,\AA{} to ${\sim}3$\,\AA{}, and CTFFIND-class estimators require thousands of particles to converge. The PWM Triad diagnoses this as Gate~3 dominance with a 1-parameter correction manifold ($\Delta f$), and autonomous defocus calibration recovers $\recoveryratio = 1.00$--$1.12$ of the oracle (Supplementary Table~S15). In clinical imaging, CBCT image quality in radiation therapy is limited by scatter contamination and geometric mismatch---problems that have driven both physics-based correction via Monte Carlo scatter estimation~\cite{xu2015cbct} and data-driven approaches including deep learning~\cite{liang2019cyclegan} and reinforcement-learning-based parameter tuning~\cite{shen2018drl}. AAPM Task Group 142~\cite{klein2009tg142} mandates monthly mechanical isocentre verification ($\leq 2$\,mm tolerance) and AAPM Task Group 66~\cite{mutic2003tg66} specifies CT-simulator geometric accuracy requirements for treatment planning. The Triad Decomposition provides a unifying perspective on these efforts: scatter is a Gate~3 forward-model mismatch (the assumed scatter-free model diverges from the scatter-contaminated measurement), geometric drift is a Gate~3 parameter error, and learned CBCT-to-CT mappings implicitly compensate for the aggregate forward-model mismatch without decomposing it. Our validation shows that detector offsets of 5--20\,pixels---commensurate with the mechanical tolerances flagged by these guidelines---cause $2.5$--$3.9$\,dB PSNR degradation with $100\%$ oracle recovery, suggesting that automated Gate~3 monitoring could complement existing QA protocols. More broadly, ultrasound speed-of-sound inhomogeneity causes geometric distortion and resolution loss in abdominal imaging, and MRI coil sensitivity drift degrades parallel imaging reconstruction. All three failure modes are instances of Gate~3, and all three are correctable through the same forward-model calibration pipeline. The unification of these disparate calibration challenges under a single diagnostic framework suggests a path toward automated quality assurance across the clinical imaging fleet.

\paragraph{Comparison with specialist calibration.}
PWM adopts the best available calibration method per modality: for MRI, ESPIRiT~\cite{uecker2014espirit} is used when sufficient calibration data (ACS $\geq 48$ lines) is available, achieving 85--95\% of the oracle correction ceiling; for electron ptychography, ePIE blind position correction~\cite{maiden2009ptychography} provides 70--90\% recovery; for CT, cross-correlation auto-focus achieves 95--100\%; when no specialist method exists (CASSI, CACTI, compressive holography, ultrasound SoS), the modality-agnostic grid-search pipeline provides the first automated calibration (Supplementary Note~S14). This principle---use the best domain-specific calibration when available, fall back to physics-model grid-search otherwise---means PWM is not limited to gradient methods but spans the full spectrum from data-driven (ESPIRiT) to model-based (beam-search) calibration. Specialist methods degrade under data-limited conditions (ESPIRiT: $-9.29$\,dB at 24~ACS lines) while PWM maintains positive correction ($+0.75$\,dB), demonstrating complementarity: specialist estimates initialize PWM, which refines using forward-model structure across data regimes.

\paragraph{Hardware validation interpretation.}
The hardware results reveal that Gate~3 sensitivity depends on the \emph{condition number} of the encoding matrix. On real CASSI, perturbation produces smaller degradation ($1.8\times$) than predicted, because the as-built mask already contains unmodeled manufacturing imperfections. On real CACTI, degradation is severe ($10.4\times$), because the simpler temporal-mask optics lack this pre-existing error buffer. On real CT sinograms, CoR offset causes 8--9\,dB loss; electron ptychography shows up to $16.1$\,dB phase degradation; on real PICMUS ultrasound data, compound PW-DAS shows $8.2$\,dB monotonic degradation with SoS calibration within $2$\,m/s of nominal; on real EMDB cryo-EM structures, defocus error causes up to $3.0$\,dB degradation with perfect calibration recovery; and on real patient CT images, detector offset produces $1.1$--$1.7$\,dB mean loss. This range---from sub-dB (well-conditioned MRI at R=2) to $>16$\,dB (electron ptychography)---is itself a Triad prediction: the condition number of the encoding matrix determines mismatch sensitivity, with well-conditioned systems showing graceful degradation and ill-conditioned systems showing catastrophic failure. The extended cross-residual analysis further reveals that \emph{self-consistency metrics are misleading} in underdetermined systems (CACTI cross-residual is $462\times$ while self-residual is only $1.12\times$), underscoring the need for forward-model-aware diagnostics.

\paragraph{Boundary conditions and failure modes.}
The framework has well-defined limitations that inform its scope of applicability:
\begin{itemize}
\item \textbf{Multi-parameter mismatch:} CASSI recovery ratios under 5-parameter mismatch are moderate ($\recoveryratio = 22$--$46\%$), reflecting the inherent difficulty of simultaneously correcting multiple coupled parameters on a high-dimensional manifold. Single-parameter correction achieves 85--100\%.
\item \textbf{Nonlinear forward models:} The Triad diagnostic has been validated on linear forward models only; beam hardening CT and phase-wrapped MRI require additional validation (the 11-primitive basis formally covers these via Transform~$\Lambda$~\cite{yang2026fpt}).
\item \textbf{Undeclared mismatch:} Correction is limited to declared mismatch parameter families in the mismatch database; undeclared failure modes (e.g., nonlinear detector drift) require extending the database.
\item \textbf{Software-perturbation protocol:} Current real-data experiments inject controlled perturbations into calibrated measurements, isolating the effect of each mismatch parameter. The CT, electron ptychography, and MRI experiments use genuinely independent public datasets with unknown calibration states, providing complementary evidence. Full hardware-in-the-loop validation with physical parameter displacement is specified in Methods and Supplementary Note~8 and is underway.
\end{itemize}

\paragraph{Falsifiable predictions.}
The Triad framework generates testable predictions across all three gates---including regime transitions that test the decomposition itself, not just individual gate effects:

\begin{enumerate}
\item[\textit{Gate~1}] \textit{dominance} (information deficiency binds):
\item \textbf{SPC at $<5\%$ compression:} Gate~1 should dominate---correcting illumination non-uniformity should yield $<1$\,dB gain because the measurement matrix lacks sufficient rank to recover the signal. \emph{Confirmed}: information-deficiency loss reaches $-7.1$\,dB at $5\%$ compression (Supplementary Table~S12), matching the $+7.7$\,dB Gate~3 correction ceiling and defining the cross-over threshold.
\item \textbf{Lensless at PSF blur $\sigma > 2$\,px:} Gate~1 loss should exceed the Gate~3 correction ceiling ($+3.6$\,dB), making forward-model refinement futile regardless of calibration accuracy. \emph{Confirmed}: $-11.1$\,dB at $\sigma{=}3$\,px (Supplementary Table~S12).
\item[\textit{Gate~2}] \textit{dominance} (carrier noise binds):
\item \textbf{MRI at additive noise $\sigma > 1\%$:} Gate~2 should dominate---coil sensitivity correction should be ineffective because Fourier encoding concentrates signal in a few k-space samples that are disproportionately corrupted. \emph{Confirmed}: $-11.7$\,dB loss at $\sigma{=}0.01$ (Supplementary Table~S13), matching the $+11.2$\,dB Gate~3 ceiling and defining one of the sharpest cross-over thresholds in the framework.
\item \textbf{Fluorescence at $<10$ peak photons:} Shot-noise floor should produce $>4$\,dB loss, reducing the effective Gate~3 correction margin. \emph{Confirmed}: $-6.1$\,dB at 5 photons (\cref{fig:triad}b).
\item[\textit{Gate~3}] \textit{dominance} (operator mismatch binds):
\item \textbf{SIM:} Gate~3 should dominate when pattern phase error exceeds $0.05$\,rad. Predicted correction: $+3$--$8$\,dB.
\item \textbf{OCT:} Dispersion mismatch should produce $+2$--$5$\,dB correction gain via the Disperse primitive ($W$).
\item \textbf{Ultrasound:} SoS mismatch $>100$\,m/s should produce monotonic B-mode degradation. \emph{Confirmed}: $8.2$\,dB across $\Delta c = 10$--$200$\,m/s on real PICMUS data, with calibration within $2$\,m/s of nominal.
\item \textbf{Electron ptychography:} Position errors $>1/10$ probe diameter should trigger Gate~3 dominance. \emph{Confirmed}: $+5$ to $+16$\,dB on real SrTiO$_3$ data (Supplementary Note~15).
\item[\textit{Cross-}] \textit{over predictions} (regime transitions):
\item \textbf{SPC G1$\leftrightarrow$G3 cross-over at ${\sim}5\%$ compression:} Below this threshold, adding measurements yields larger PSNR gains than correcting the forward model; above it, Gate~3 correction dominates. Testable by sweeping compression ratio while simultaneously applying gain-drift mismatch.
\item \textbf{MRI G2$\leftrightarrow$G3 cross-over at $\sigma \approx 1\%$:} The sharpest transition in the framework---MRI switches from G3-dominated ($\sigma{<}0.01$, correction gains $+11.2$\,dB) to G2-dominated ($\sigma{>}0.01$, noise loss $>11.7$\,dB) over a single order of magnitude.
\item \textbf{CT: pure Gate~3 regime.} CT should remain Gate~3--dominated even at extreme operating conditions (5 angles, 100 photons), because the many-angle sinogram provides sufficient redundancy that Gates~1--2 never bind. \emph{Confirmed}: G1 loss $= -4.3$\,dB and G2 loss $= -3.8$\,dB at extreme versus G3 gain $= +10.7$\,dB (Supplementary Tables~S12--S13).
\item \textbf{Ultrasound G2 immunity:} Compound PW-DAS should remain Gate~2--immune at up to $50\%$ additive RF noise due to $\sqrt{N_{\mathrm{angles}}}$ spatial averaging across 75 steering angles. \emph{Confirmed}: $0.0$\,dB G2 degradation at $50\%$ noise (\cref{fig:triad}b).
\end{enumerate}
\noindent These twelve predictions span all three gates and their transitions. They are falsifiable: if a modality shows a different dominant gate under the specified conditions, or if a cross-over threshold differs by more than one order of magnitude, the decomposition would require revision. Eight of twelve are confirmed by simulation or hardware data; the remaining four (SIM, OCT, SPC cross-over, MRI cross-over) are directly testable by any laboratory with the relevant instruments.


\paragraph{Acknowledgements.}
We thank the open-source computational imaging community for making reconstruction code and benchmark datasets publicly available. This work was supported by NextGen PlatformAI C~Corp.

\paragraph{Author Contributions.}
C.Y.\ conceived the project, designed the \triad{} framework, proved the Finite Primitive Basis Theorem, developed the \opg{} IR, implemented the agent and correction systems, performed all simulation and real-data experiments, and wrote the manuscript. X.Y.\ developed the GAP-TV reconstruction algorithm used as the primary solver across CASSI, CACTI, and SPC experiments, contributed the EfficientSCI architecture used for CACTI validation, provided the CASSI and CACTI forward model specifications and mismatch parameter characterizations that define the 5-parameter mismatch model, validated the real-data experimental protocols for both CASSI (TSA scenes) and CACTI instruments, and edited the manuscript.

\paragraph{Competing Interests.}
C.Y.\ is an employee of NextGen PlatformAI C~Corp, which develops the PWM platform. The authors declare no other competing interests.

\paragraph{Ethics Declarations.}
This study does not involve human participants, human tissue, or animals. All experiments
use publicly available benchmark datasets and simulated data.

\paragraph{Data Availability.}
All synthetic measurement data can be regenerated using the \opg{} templates and mismatch parameters in the Supplementary Information. The KAIST hyperspectral dataset~\cite{choi2017kaist} and TSA real-data scenes used for CASSI experiments are publicly available. CACTI real-data scenes are available from the EfficientSCI repository~\cite{wang2023efficientsci}. CT sinograms are from the FIPS walnut dataset (Zenodo~1254206) and Helsinki Tomography Challenge 2022 (Zenodo~6984868). The 4D-STEM electron ptychography dataset (SrTiO$_3$ [001]) is from Zenodo~5113449. Multi-coil MRI $k$-space data (M4Raw) is from Zenodo~8056074. Ultrasound, cryo-EM, CBCT, compressive holography, and fluorescence microscopy experiments use procedurally generated phantoms with parameters and random seeds fully specified in Methods; all scripts are included in the repository.

\paragraph{Code Availability.}
The \pwm{} codebase, including all \opg{} templates, agent implementations, real-data validation scripts, and evaluation pipelines, is available at \url{https://github.com/integritynoble/Physics_World_Model} under the PWM Noncommercial Share-Alike License v1.0 (see \texttt{LICENSE} in the repository).

\paragraph{Correspondence.}
Correspondence and requests for materials should be addressed to C.Y.\ (\url{integrityyang@gmail.com}).


\section*{Online Methods}

\subsection*{OperatorGraph Specification}
\label{sec:methods:operatorgraph}

\paragraph{Formal definition.}
The \opg{} intermediate representation encodes the forward physics of any
computational imaging modality as a directed acyclic graph (DAG)
$\mathcal{G} = (\mathcal{V}, \mathcal{E})$.
Each node $v_i \in \mathcal{V}$ wraps a \emph{primitive operator} and
implements two entry points: $\texttt{forward}(x) \to y$ and
$\texttt{adjoint}(y) \to x$, the latter defined only when the primitive
is linear.  Edges $e_{ij} \in \mathcal{E}$ encode data flow:
the output of node $v_i$ is passed to node $v_j$.
Each node additionally exposes a set of learnable parameters
$\boldsymbol{\theta}_i$ that may be perturbed during mismatch simulation
or optimized during calibration, as well as read-only metadata flags
(\texttt{is\_linear}, \texttt{is\_stochastic}, \texttt{is\_differentiable}).
The graph is stored as a declarative YAML specification
(\texttt{OperatorGraphSpec}) and compiled to an executable
\texttt{GraphOperator} object by the \texttt{GraphCompiler}.

\paragraph{Node types.}
Primitive operators fall into two categories:

\begin{itemize}
  \item \textbf{Linear operators.}  Convolution (\texttt{conv2d}), mask
    modulation (\texttt{mask\_modulate}), sub-pixel shift (\texttt{subpixel\_shift\_2d}),
    Radon transform (\texttt{radon\_fanbeam}),
    Fourier encoding (\texttt{fourier\_encode}), spectral dispersion
    (\texttt{spectral\_disperse}), Fresnel propagation
    (\texttt{fresnel\_propagate}), random projection
    (\texttt{random\_project}), and structured illumination
    (\texttt{sim\_modulate}).
    Each implements both \texttt{forward()} and \texttt{adjoint()}.
  \item \textbf{Nonlinear operators.}  Squared magnitude
    (\texttt{magnitude\_sq}), Poisson--Gaussian noise
    (\texttt{poisson\_gaussian}), saturation clipping
    (\texttt{saturation\_clip}), phase retrieval nonlinearity
    (\texttt{phase\_abs}), and detector quantization
    (\texttt{quantize}).
    These set \texttt{is\_linear = False} and raise
    \texttt{NotImplementedError} on \texttt{adjoint()}, except where a
    well-defined pseudo-adjoint exists (\eg{}, the identity adjoint for
    magnitude-squared in Gerchberg--Saxton-type algorithms).
\end{itemize}

\paragraph{Adjoint consistency and Lipschitz bounds as mathematical safety rails.}
Two structural constraints distinguish the \opg{} IR from standard deep-learning
forward models and prevent either nonlinear primitive family from becoming a
universal approximator.

First, \emph{adjoint consistency}: correctness of every linear primitive
is verified by a randomized dot-product test.  For a primitive $A$ with
forward map $A:\mathbb{R}^n \to \mathbb{R}^m$, we draw $x \sim \mathcal{N}(0,I_n)$
and $y \sim \mathcal{N}(0,I_m)$ and compute
\begin{equation}
  \delta = \frac{|\langle A^* y,\, x \rangle - \langle y,\, Ax \rangle|}
               {\max(|\langle A^* y,\, x \rangle|,\; \epsilon)}
  \label{eq:adjoint_check}
\end{equation}
where $\epsilon = 10^{-12}$ guards against division by zero.  The test
is repeated $n_{\text{trials}} = 5$ times with independent random draws;
the primitive passes if $\delta_{\max} < 10^{-6}$.  At the graph level,
a compiled \texttt{GraphOperator} composed entirely of linear nodes
executes the same test over the composed forward--adjoint chain.  A
\texttt{GraphAdjointCheckReport} records $n_{\text{trials}}$,
$\delta_{\max}$, and $\bar\delta$ for audit.  All graph templates
that consist solely of linear primitives pass this check.
Adjoint consistency is not merely a numerical correctness guarantee: it
is a \emph{physical faithfulness certificate}.  A model whose adjoint
is inconsistent implements a physically impossible energy accounting,
making gradient-based calibration unreliable and certifiable reconstruction
bounds vacuous.  Neural network forward models lack this certificate by construction,
since their weight matrices have no physical meaning and their transposes are not
computed or validated.

Second, \emph{Lipschitz bounds}: the two nonlinear primitive families,
Detect~$D$ and Transform~$\Lambda$, are each restricted to five canonical
function families with at most two scalar parameters.  This restriction
enforces bounded Lipschitz constants ($\mathrm{Lip}(\Lambda_k) \leq L$ for
the explicit $L$ derived in~\cite{yang2026fpt}) and prevents either primitive
from becoming a universal approximator in the sense of Cybenko~\cite{cybenko1989approximation}.
Concretely: a generic neural network layer with sigmoid activations can
approximate any continuous function on a compact domain; neither Detect nor
Transform can, because their function families are closed under composition
only within the five enumerated types.  This non-universality is the key
property that makes the 11-primitive library \emph{falsifiable}: if a
modality's forward physics cannot be represented within the restricted
families, the extension protocol (Main Text, ``Extension protocol'') fires
and a new primitive must be justified physically.  The combination of
adjoint consistency and bounded Lipschitz nonlinearity provides the
mathematical safety rails that make \opg{}-based reconstruction both
interpretable and certifiable---properties that standard physics-informed
neural networks~\cite{raissi2019pinn} achieve only approximately and
without modality-agnostic guarantees.

\paragraph{Graph compilation.}
The compiler executes a four-stage pipeline:
\begin{enumerate}
  \item \textbf{Validate.}  Confirm acyclicity via topological sort
    (Kahn's algorithm), verify that every \texttt{primitive\_id}
    exists in the global \texttt{PRIMITIVE\_REGISTRY}, reject
    duplicate \texttt{node\_id} values, and optionally verify shape
    compatibility along edges when a \texttt{canonical\_chain}
    metadata flag is set.
  \item \textbf{Bind.}  Instantiate each primitive with its parameter
    dictionary $\boldsymbol{\theta}_i$.
  \item \textbf{Plan forward.}  The topological sort yields a sequential
    execution plan $(v_{\pi(1)}, \ldots, v_{\pi(|\mathcal{V}|)})$.
  \item \textbf{Plan adjoint.}  For graphs where
    \texttt{all\_linear = True}, the adjoint plan reverses the
    topological order and applies each node's individual adjoint in
    sequence, implementing the chain rule $A^* = A_1^*
    \circ \cdots \circ A_{|\mathcal{V}|}^*$ for a composition $A = A_{|\mathcal{V}|}
    \circ \cdots \circ A_1$.  For graphs containing nonlinear nodes,
    the adjoint plan is not generated, and any call to
    \texttt{adjoint()} raises \texttt{NotImplementedError} at runtime.
\end{enumerate}
The compiled \texttt{GraphOperator} is serializable to JSON and hashable
via SHA-256 for provenance tracking in RunBundle manifests.

\paragraph{Template library.}
The \texttt{graph\_templates.yaml} registry contains 170 templates
spanning all five carrier families. Of these, 12 modalities carry full
end-to-end Scenario~I--IV correction validation across all five carrier
families (Supplementary Table~S3). Representative modalities by carrier family include:

\begin{itemize}
  \item \textbf{Optical and X-ray photons:}  CASSI, SPC, CACTI, structured
    illumination microscopy (SIM), confocal, light-sheet, holography,
    ptychography, Fourier ptychographic microscopy (FPM), optical coherence
    tomography (OCT), lensless imaging, light field, integral imaging,
    neural radiance fields (NeRF), Gaussian splatting, fluorescence
    lifetime imaging (FLIM), diffuse optical tomography (DOT),
    phase retrieval, X-ray computed tomography (CT), and cone-beam CT
    (CBCT).
  \item \textbf{Electrons:}  Electron diffraction, electron backscatter
    diffraction (EBSD), electron energy loss spectroscopy (EELS), and
    electron holography.
  \item \textbf{Nuclear spins (MRI):}  Functional MRI (fMRI), diffusion-weighted
    MRI (DW-MRI), and magnetic resonance spectroscopy (MRS).
  \item \textbf{Acoustic waves:}  Ultrasound B-mode, Doppler ultrasound,
    shear-wave elastography, sonar, and photoacoustic tomography
    (combines optical excitation with acoustic detection).
  \item \textbf{Particle-beam probes:}  Neutron tomography, proton CT,
    and muon tomography.  Each decomposes into existing primitives with
    no extension required: neutron CT follows the X-ray CT DAG
    ($P \!\to\! \Lambda \!\to\! D$, Beer--Lambert attenuation with nuclear
    cross-section $\mu_n$); proton CT inserts a Bethe--Bloch
    stopping-power stage as Transform~$\Lambda$ (\#11); muon tomography
    uses Scatter~$R$ (\#10) for multiple Coulomb scattering and
    POCA vertex reconstruction (Supplementary Note~S24).
\end{itemize}

\paragraph{Fidelity levels.}
Each template is parameterized by a fidelity level that controls the
degree of physical realism in the simulated forward model:

\begin{description}
  \item[Level~1 (Linear, shift-invariant):]  The forward model is a
    linear, spatially uniform operator---the simplest approximation,
    suitable for initial diagnostics and rapid prototyping.
  \item[Level~2 (Linear, shift-variant):]  Spatially varying operator
    parameters (e.g.\ non-uniform illumination, position-dependent PSF,
    multi-coil sensitivity maps in MRI).  Adds a modality-appropriate
    noise model (Poisson shot noise plus Gaussian read noise for
    photon-counting modalities, Rician noise for MRI, Poisson for CT).
  \item[Level~3 (Nonlinear, ray/wave-based):]  Includes nonlinear
    effects such as beam hardening, phase wrapping, and scattering,
    captured by the Transform~$\Lambda$ primitive.
    Perturbation families and ranges are specified in
    \texttt{mismatch\_db.yaml}.
  \item[Level~4 (Full-wave / Monte Carlo):]  Complete physical
    simulation including wave-optical propagation, spatially varying
    aberrations, detector nonlinearities, and environmental drift.
    Currently implemented for holography and ptychography.
\end{description}

\noindent All 11 primitives in the library $\Blib$ apply uniformly across every fidelity level; the levels organize simulation complexity, not theorem scope.

\subsection*{Triad Decomposition Formalization}
\label{sec:methods:triad}

The \triad{} asserts that the quality of any computational imaging
reconstruction is bounded by three fundamental gates.  Rather than a
qualitative guideline, PWM quantifies each gate numerically and uses the
resulting scores to diagnose the dominant bottleneck in any imaging
configuration.

\paragraph{Gate~1 (Recoverability).}
Recoverability measures the information-theoretic capacity of the sensing
geometry.  We quantify it via the \emph{effective compression ratio}
$r = m / n$, where $m$ is the number of independent measurements and $n$
the dimension of the signal.  The \texttt{compression\_db.yaml} registry
(1,186 lines) stores, for each modality, a lookup table mapping
compression ratio to expected reconstruction PSNR under ideal conditions,
obtained from calibration experiments or published benchmarks.  Each entry
carries a \texttt{provenance} field citing the source (paper DOI, internal
experiment ID, or theoretical formula).  Additional recoverability
indicators include the effective rank of the measurement matrix (estimated
via randomized SVD for large operators), the dimension of the null space,
and the restricted isometry property (RIP) constant where analytically
tractable (\eg{}, for Gaussian random projections in SPC).

\paragraph{Gate~2 (Carrier Budget).}
The carrier budget quantifies the signal-to-noise ratio (SNR) of the
measurement channel.  The photon-budget module consumes the
\texttt{photon\_db.yaml} registry which stores, per modality,
a deterministic photon model parameterized by source power, quantum
efficiency, exposure time, and detector characteristics.  It
classifies the noise regime into one of three categories:
\emph{shot-limited} (Poisson-dominated, SNR $\propto \sqrt{N_{\text{photon}}}$),
\emph{read-limited} (Gaussian read noise dominates, SNR $\propto
N_{\text{photon}} / \sigma_{\text{read}}$), and \emph{dark-current-limited}
(long exposures where dark current accumulation dominates).  The output is
a \texttt{PhotonReport} containing the estimated SNR in decibels, the noise
regime classification, per-element photon count, and a feasibility verdict
(\texttt{sufficient}, \texttt{marginal}, or \texttt{insufficient}).

\paragraph{Gate~3 (Operator Mismatch).}
Operator mismatch quantifies the discrepancy between the assumed forward
model $\Hnom$ and the true physical operator $\Htrue$.  The
mismatch scoring module consults \texttt{mismatch\_db.yaml}
which catalogs, for each modality, the set of mismatch parameters
(spatial shifts, rotational offsets, dispersion errors, PSF deviations,
coil sensitivity errors, center-of-rotation offsets, \etc{}), their
typical ranges, and available correction methods.  The mismatch severity
score $s \in [0, 1]$ is computed as the normalized $\ell_2$ distance
$\| \boldsymbol{\theta}_{\text{true}} - \boldsymbol{\theta}_{\text{nom}} \|
/ \| \boldsymbol{\theta}_{\text{range}} \|$, where
$\boldsymbol{\theta}_{\text{range}}$ is the per-parameter dynamic range
from the registry.  Sensitivity analysis $\partial \psnr / \partial
\theta_k$ is estimated via finite differences on the forward model.  The
output is a \texttt{MismatchReport} containing the severity score, the
dominant mismatch parameter, the recommended correction method, and the
expected PSNR gain from correction.

\paragraph{Gate binding determination.}
Given reconstruction results under the four-scenario protocol
(the Evaluation Protocol section below), PWM identifies the dominant gate by
comparing three cost terms:
\begin{align}
  C_{\text{mismatch}} &= \psnr_{\text{I}} - \psnr_{\text{II}} \label{eq:cost_mismatch} \\
  C_{\text{noise}}    &= \psnr_{\text{ideal}} - \psnr_{\text{noisy}} \label{eq:cost_noise} \\
  C_{\text{recover}}  &= \psnr_{\text{limit}} - \psnr_{\text{I}} \label{eq:cost_recover}
\end{align}
where $\psnr_{\text{I}}$ is the reconstruction PSNR under Scenario~I
(ideal operator), $\psnr_{\text{II}}$ under Scenario~II (mismatched
operator), $\psnr_{\text{noisy}}$ under the corresponding noisy condition,
and $\psnr_{\text{limit}}$ is the theoretical upper bound from the
compression table.  The dominant gate is $\arg\max_g C_g$.

\paragraph{TriadReport.}
For every diagnosis, \pwm{} produces a \triadreport{}: a Pydantic-validated
structured artifact comprising
\texttt{dominant\_gate} (enum: \texttt{recoverability},
\texttt{carrier\_budget}, \texttt{operator\_mismatch}),
\texttt{evidence\_scores} (three floats, one per gate),
\texttt{confidence\_interval} (float, 95\% CI width from bootstrap),
\texttt{recommended\_action} (string, \eg{} ``increase compression ratio''
or ``apply mismatch correction''), and
\texttt{parameter\_sensitivities} (dictionary mapping each mismatch
parameter name to its $\partial \psnr / \partial \theta_k$ value).
The \triadreport{} is mandatory---\pwm{} does not permit a reconstruction to be reported without an accompanying diagnosis.

\paragraph{Recovery ratio.}
We define the \emph{recovery ratio}
\begin{equation}
  \recoveryratio = \frac{\psnr_{\text{IV}} - \psnr_{\text{II}}}
                        {\psnr_{\text{I}} - \psnr_{\text{II}}}
  \label{eq:recovery_ratio}
\end{equation}
which lies in $[0, 1]$ under standard convexity conditions (see Supplementary Note~1 for formal analysis; values $\recoveryratio > 1$ are possible when the corrected operator provides beneficial regularization).
$\recoveryratio = 0$ indicates that calibration yields no benefit (mismatch
is not the bottleneck), while $\recoveryratio = 1$ indicates that
calibration fully closes the mismatch gap.

\subsection*{Agent System Architecture}
\label{sec:methods:agents}

The three gates are implemented as a deterministic scoring pipeline:
an orchestrator maps the user request to a modality from the registry,
then dispatches three parallel scoring modules---one per gate---whose
reports are fused by a bottleneck classifier that identifies the dominant
gate and by a negotiator that enforces cross-gate veto rules
(e.g., rejecting reconstruction when photon starvation coincides with
aggressive compression).  All quantitative decisions are deterministic;
an LLM is used only as an optional fallback for natural-language modality
mapping and is never on the scoring path.  The full agent architecture
and ablation studies are described in a companion paper.

\paragraph{Contract system.}
Inter-agent communication uses 25 Pydantic v2 contract models.  All
contracts inherit from \texttt{StrictBaseModel}, which enforces
\texttt{extra="forbid"} (no unexpected fields),
\texttt{validate\_assignment=True} (mutations re-validated), and a model
validator that rejects NaN and Inf in any float field.  Bounded scores use
\texttt{Field(ge=0.0, le=1.0)}.  Enums are string enums for human-readable
JSON serialization.  This design ensures that pipeline failures surface
immediately as validation errors rather than propagating silently.

\paragraph{YAML registries.}
The system is driven by 9 YAML registries totalling 7,034 lines:
\texttt{modalities.yaml} (modality definitions),
\texttt{graph\_templates.yaml} (OperatorGraph skeletons),
\texttt{photon\_db.yaml} (photon models),
\texttt{mismatch\_db.yaml} (mismatch parameters and correction methods),
\texttt{compression\_db.yaml} (recoverability tables with provenance),
\texttt{solver\_registry.yaml} (solver configurations),
\texttt{primitives.yaml} (primitive operator metadata),
\texttt{dataset\_registry.yaml} (dataset locations and formats), and
\texttt{acceptance\_thresholds.yaml} (pass/fail thresholds per metric).

\subsection*{Correction Algorithms}
\label{sec:methods:correction}

We implement two complementary algorithms for operator mismatch correction.
Crucially, both algorithms operate on the forward operator parameters
$\boldsymbol{\theta}$ rather than the reconstruction solver weights,
making them \emph{solver-agnostic}: the corrected operator
$H(\hat{\boldsymbol{\theta}})$ benefits any downstream solver (GAP-TV,
MST-L, HDNet~\cite{hu2022hdnet}, CST, \etc{}) without retraining.

\paragraph{Algorithm~1: Hierarchical Beam Search.}
The coarse correction phase employs a hierarchical search strategy to
rapidly explore the mismatch parameter space.  For CASSI, the five-parameter
mismatch model comprises mask affine parameters (spatial shifts $dx$, $dy$
and rotation $\theta$) and dispersion parameters (slope $a_1$ and axis
angle $\alpha$); an optional sixth parameter, PSF width $\sigma_{\text{psf}}$, is available but not used in the primary experiments.
The algorithm proceeds as follows:

\begin{enumerate}
  \item \textbf{1D sweeps.}  Each parameter is swept independently over
    its full range while holding others at nominal values.  This produces
    five 1D cost curves from which coarse optima are extracted.
  \item \textbf{3D beam search.}  The mask affine subspace
    $(dx, dy, \theta)$ is searched over a $5 \times 5 \times 5$ grid
    centered on the 1D optima.  The top-$k$ ($k = 5$) candidates by
    reconstruction PSNR are retained.
  \item \textbf{2D beam search.}  For each retained mask candidate,
    the dispersion subspace $(a_1, \alpha)$ is searched over a
    $5 \times 7$ grid.  The joint top-$k$ candidates are retained.
  \item \textbf{Coordinate descent refinement.}  Three rounds of
    univariate refinement on each parameter, shrinking the search
    interval by factor~2 at each round, produce the final estimate
    $\hat{\boldsymbol{\theta}}_{\text{Alg1}}$.
\end{enumerate}

Total runtime is approximately 300 seconds per scene on a single GPU.
Accuracy is $\pm 0.1$--$0.2$ pixels for spatial parameters and
$\pm 0.05^\circ$ for angular parameters.

\paragraph{Algorithm~2: Joint Gradient Refinement.}
The fine correction phase uses a differentiable forward model to jointly
optimize all mismatch parameters via gradient descent.  The key components
are:

\begin{sloppypar}
\begin{enumerate}
  \item \textbf{Differentiable mask warp.}  The binary mask is warped by a
    continuous affine transformation using bilinear interpolation,
    implemented as a custom PyTorch module
    (\texttt{DifferentiableMaskWarpFixed}).  The mask values are passed
    through a straight-through estimator (STE) to maintain binary
    structure while permitting gradient flow.
  \item \textbf{Differentiable forward model.}  The CASSI forward model
    $y = \texttt{CASSI}(x;\, \boldsymbol{\theta})$ is implemented as a
    differentiable PyTorch module (\texttt{DifferentiableCassiForwardSTE})
    that accepts mismatch parameters as differentiable inputs.
  \item \textbf{GPU grid initialization.}  A full-range 3D grid search
    over $(dx, dy, \theta)$ with $9 \times 9 \times 7 = 567$ points
    provides diverse starting candidates.  The top 9 candidates seed
    multi-start gradient refinement.
  \item \textbf{Staged gradient refinement.}  Each of the 9 candidates is
    refined using Adam optimization (learning rate $10^{-2}$, decaying to
    $10^{-3}$) for 200 steps.  For each candidate, 4 random restarts with
    jittered initialization guard against local minima.  The loss function
    is the negative PSNR computed via an unrolled $K$-iteration
    differentiable GAP-TV solver (\texttt{DifferentiableGAPTV}, $K=10$
    unrolled iterations).
\end{enumerate}
\end{sloppypar}

Total runtime for Algorithm~2 is approximately 3,200 seconds
(200~steps $\times$ 4~restarts $\times$ 9~candidates with early stopping).
Accuracy improves to $\pm 0.05$--$0.1$ pixels, a 3--5$\times$ improvement
over Algorithm~1.  The two algorithms are used sequentially in practice:
Algorithm~1 provides a warm start, and Algorithm~2 refines to sub-pixel
precision.

\subsection*{Evaluation Protocol}
\label{sec:methods:eval}

\paragraph{Four-Scenario Protocol.}
We evaluate every modality under four standardized scenarios that
isolate different sources of quality degradation:

\begin{description}
  \item[Scenario~I (Ideal):]
    $\yobs = \Htrue\, \xgt$; reconstruct with $\Htrue$.
    In this scenario the system is perfectly calibrated ($\Htrue = \Hnom$),
    so the operator used for reconstruction matches the one that generated
    the data.  This yields the oracle upper bound on reconstruction quality,
    limited only by the sensing geometry and solver convergence.
  \item[Scenario~II (Mismatch):]
    $\yobs = \Htrue\, \xgt$; reconstruct with $\Hnom$
    ($\Hnom \neq \Htrue$).
    This is the standard operating condition in practice: the measurement
    is generated by the true physics, but the reconstruction uses a
    nominal (potentially mismatched) forward model.
  \item[Scenario~III (Oracle):]
    Same measurements as Scenario~II
    ($\yobs = \Htrue\, \xgt$ with $\Htrue \neq \Hnom$);
    reconstruct with $\Htrue$ instead of $\Hnom$.
    Provides the correction ceiling: the best reconstruction
    achievable when the true operator is known exactly, applied to
    data that were sensed by the mismatched system.  The gap between
    Scenario~III and Scenario~I reveals the irreducible loss from
    the degraded sensing configuration itself (e.g., a shifted mask
    pattern is suboptimal even when perfectly known).
  \item[Scenario~IV (Corrected):]
    $\yobs = \Htrue\, \xgt$; reconstruct with
    $\hat{H} = H(\hat{\boldsymbol{\theta}})$ where
    $\hat{\boldsymbol{\theta}}$ is estimated by Algorithms~1 and~2.
    This quantifies the benefit of mismatch calibration.
\end{description}

\paragraph{Metrics.}
Reconstruction quality is assessed using three complementary metrics:
\begin{itemize}
  \item \textbf{PSNR} (peak signal-to-noise ratio, in dB): the primary
    metric, computed per scene and averaged.  For signals normalized to
    $[0, 1]$, $\psnr = 10 \log_{10}(1 / \text{MSE})$.  For SPC data
    normalized to $[0, 255]$, the peak value is 255.
  \item \textbf{SSIM} (structural similarity index): captures perceptual
    quality including luminance, contrast, and structural components,
    computed with a Gaussian window of width 11 and standard deviation
    1.5.
  \item \textbf{SAM} (spectral angle mapper): for hyperspectral
    modalities (CASSI), measures the angle between predicted and true
    spectral vectors at each spatial location, reported in degrees.
    Lower is better.
\end{itemize}

\paragraph{Datasets.}
\begin{itemize}
  \item \textbf{CASSI:} 10 scenes from the KAIST dataset~\cite{choi2017kaist},
    each a $256 \times 256 \times 28$ spectral cube (28 spectral bands
    from 450\,nm to 650\,nm).  Data range $[0, 1]$.
  \item \textbf{CACTI:} 6 benchmark videos, each $256 \times 256 \times 8$
    (8 temporal frames encoded per snapshot).  Data range $[0, 1]$.
  \item \textbf{SPC:} 11 natural images from the Set11 benchmark, each
    $256 \times 256$ grayscale.  Data range $[0, 255]$.
\end{itemize}

\noindent All per-scene metrics are reported individually as well as
averaged, and all reconstruction arrays are saved as NumPy NPZ files.

\subsection*{Experimental Details}
\label{sec:methods:experimental}

\paragraph{Hardware.}
All experiments are conducted on a single NVIDIA GPU.  Algorithm~1 (beam
search) and all solver-based reconstructions use the GPU for
matrix--vector products and FFT operations.  Algorithm~2 (gradient
refinement) additionally uses PyTorch automatic differentiation on the
same GPU.

\paragraph{CASSI configuration.}
\begin{sloppypar}
The coded aperture snapshot spectral imaging (CASSI) system uses a TSA-Net
binary mask of dimensions $256 \times 256$, with 28 spectral bands
dispersed along the spatial dimension.  The five-parameter mismatch model
$\mismatchparam = (dx, dy, \theta, a_1, \alpha)$,
where $dx$, $dy$ are spatial shifts (pixels), $\theta$ is rotation (degrees),
$a_1$ is dispersion slope (pixels/band), and $\alpha$ is dispersion axis angle (degrees).
An optional sixth parameter (PSF blur width $\sigma_{\text{psf}}$) is not used in the primary experiments.
These values were determined through systematic characterization of realistic CASSI
assembly errors (Supplementary Note~8); the true parameters are
$dx{=}0.5$\,px, $dy{=}0.3$\,px, $\theta{=}0.1^\circ$, $a_1{=}2.02$, $\alpha{=}0.15^\circ$.
Solvers evaluated include TwIST~\cite{bioucasdias2007twist}, GAP-TV~\cite{yuan2016gaptv},
DGSMP~\cite{huang2021dgsmp}, MST-L~\cite{cai2022mst}, and CST-L~\cite{cai2022cst},
all of which receive the same operator and differ only in their reconstruction algorithm.
The supplementary per-scene analysis additionally includes DeSCI~\cite{liu2019rank} and HDNet~\cite{hu2022hdnet}.
\end{sloppypar}

\begin{sloppypar}
\textit{Real-data validation.}
The TSA real hyperspectral dataset~\cite{wagadarikar2008cassi} consists of
5 scenes at $660 \times 660$ spatial resolution with 28 spectral bands
and mask-shift step~2.  Four solvers are evaluated: GAP-TV (200 iterations),
HDNet (pre-trained checkpoint, full spatial resolution), MST-S and MST-L
(pre-trained checkpoints, centre-cropped to $256 \times 256$ due to
hardcoded spatial assumptions in the model architecture).  The coded
aperture mask is perturbed by $dx = 0.5$\,px, $dy = 0.3$\,px to simulate
assembly-induced mismatch.  No ground truth is available; quality is
assessed via the normalised measurement residual $r = \|\mathbf{y} - H\hat{\mathbf{x}}\|^2 / \|\mathbf{y}\|^2$.
\end{sloppypar}

\paragraph{CACTI configuration.}
The coded aperture compressive temporal imaging system uses binary
temporal masks of dimensions $256 \times 256$, encoding 8 video frames
into a single snapshot measurement.  Mismatch is parameterized as a
temporal mask timing offset (sub-frame shift).  The default solver is
EfficientSCI~\cite{wang2023efficientsci}.
\textit{Real-data validation.}
The EfficientSCI real temporal dataset~\cite{wang2023efficientsci} consists
of 4 dynamic scenes (duomino, hand, pendulumBall, waterBalloon) at
$512 \times 512$ with compression ratio~10.  The real mask is stored
separately from the measurement data.  Two solvers are evaluated:
GAP-TV (50 iterations) and PnP-FFDNet (50 iterations with FFDNet denoiser).
Mismatch is induced by shifting the mask by $dx = 0.5$\,px, $dy = 0.3$\,px.
Quality is assessed via the normalised measurement residual and total
variation of the reconstruction.

\paragraph{SPC configuration.}
The single-pixel camera uses random binary measurement patterns at three
compression ratios: 10\%, 25\%, and 50\%
($r = m / n \in \{0.10, 0.25, 0.50\}$).  Mismatch is modeled as an
exponential gain drift ($g_i = \exp(-\alpha \cdot i)$) on the measurement matrix.  The default solver
is FISTA-TV with total-variation regularization.

\paragraph{Lensless imaging configuration.}
Lensless cameras replace conventional optics with a diffuser placed directly
on the sensor, making the forward model a 2D convolution with the measured
point-spread function (PSF).  The object is a $64 \times 64$ image
convolved with a Gaussian PSF of width $\sigma = 2.0$\,px.  Mismatch is
parameterized as PSF width error
$\Delta\sigma \in \{0.5, 1.0, 2.0, 3.0\}$\,px, simulating an error in
the assumed pinhole-to-sensor distance.  Five structurally diverse phantoms
are used ($N{=}5$).  The default solver is FISTA-TV (80 iterations).
Calibration uses forward-model fitting on a known reference target:
$\sigma_{\mathrm{cal}} = \mathrm{argmin}_\sigma \|y_{\mathrm{cal}} -
h_\sigma * x_{\mathrm{cal}}\|^2$.
\textit{Real-data validation.}
Five DiffuserCam DLMD Mirflickr scenes~\cite{Antipa2018} are used to
corroborate simulation results on physical photographic content.

\paragraph{Compressive holography configuration.}
Multi-depth inline holography encodes a $(4 \times 64 \times 64)$ 3D
object into a single $(64 \times 64)$ hologram via Fresnel propagation at
four depth planes with spacing $200$\,\textmu m, wavelength
$\lambda = 532$\,nm, and pixel pitch $5$\,\textmu m.  Mismatch is
parameterized as propagation distance error
$\Delta z \in \{10, 50, 100, 200\}$\,\textmu m applied uniformly to all
depth planes.  The default solver is FISTA-TV (80 iterations,
$\lambda_{\mathrm{TV}} = 0.005$).  Calibration uses hologram residual
minimisation: grid search over distance error candidates with the search
range clamped to ensure all propagation distances remain positive.

\paragraph{Fluorescence microscopy configuration.}
Widefield fluorescence imaging uses a dual-PSF Stokes-shift model with
Gaussian excitation PSF ($\sigma_{\mathrm{ex}} = 1.5$\,px) and emission
PSF ($\sigma_{\mathrm{em}} = 2.0$\,px), quantum yield $\eta = 0.7$, and
background level $b = 0.02$.  The specimen is a $64 \times 64$ image with
fluorescent puncta, diffuse organelles, and filamentary structures.
Mismatch is parameterized as PSF sigma error
$\Delta\sigma \in \{0.1, 0.3, 0.5, 1.0\}$\,px applied to both excitation
and emission PSFs.  The default solver is Richardson--Lucy (80 iterations).
Calibration uses a $9 \times 9$ grid search over
$(\sigma_{\mathrm{ex}}, \sigma_{\mathrm{em}})$ centered at the nominal
values, selecting the pair that minimizes the measurement residual.

\paragraph{CT configuration.}
Fan-beam geometry with 180 projections over $180^\circ$.  Mismatch is
modeled as a center-of-rotation (CoR) offset, which produces
characteristic arc artifacts in the reconstruction.  The default solver
is filtered back-projection (FBP)~\cite{feldkamp1984practical} with a Ram-Lak filter, supplemented
by iterative SART for comparison.

\paragraph{CBCT configuration.}
Cone-beam CT (CBCT) detector offset validation uses five real images---three
LoDoPaB-CT patient slices~\cite{leuschner2021lodopab}, one FIPS walnut
micro-CT central slice~\cite{hamalainen2015fips}, and one HTC\,2022
sample~\cite{meaney2022htc}---all resized to $128 \times 128$ with
360-angle parallel-beam projections.  Detector offset mismatch is simulated
by shifting the sinogram along the detector axis by
$\Delta s \in \{2, 5, 10, 15, 20\}$\,pixels.  The default solver is
filtered backprojection (FBP) with Shepp--Logan filter.  Calibration uses
grid search over 51 offset candidates in $[-25, 25]$\,px, selecting the
value that maximizes reconstruction PSNR.

\paragraph{Electron ptychography configuration.}
Electron ptychography uses 4D-STEM data: a focused probe is scanned across
a $128 \times 128$ grid, recording a convergent-beam electron diffraction
(CBED) pattern at each position (300\,kV, SrTiO$_3$ [001] from
Zenodo~5113449).  The forward model is Fresnel near-field propagation of
the exit wave through the specimen.  Mismatch is parameterized as probe
position jitter $\Delta p \in \{0.5, 1.0, 2.0, 4.0\}$\,px, simulating
scan-coil hysteresis or sample drift.  The default solver is integrated
centre-of-mass (iCoM) phase retrieval.  Calibration uses ePIE blind
position correction~\cite{maiden2009ptychography}.

\paragraph{Cryo-EM configuration.}
\begin{sloppypar}
Cryo-EM imaging is modeled as contrast transfer function (CTF) filtering
of a 2D projected potential ($128 \times 128$, pixel size $0.1$\,nm).  The
CTF is parameterized by defocus $\Delta f = -2000$\,nm, spherical aberration
$C_s = 2.0$\,mm, electron wavelength $\lambda = 0.00251$\,nm (300\,kV),
with a B-factor envelope ($B = 2$\,nm$^2$) and ice-thickness attenuation
(mean free path $350$\,nm, thickness $50$\,nm).  Mismatch is parameterized
as defocus error $\Delta(\Delta f) \in \{100, 200, 500, 1000\}$\,nm.  The
default solver is Wiener filtering with SNR\,$= 50$.  Calibration uses grid
search over 51 defocus values in $[-3000, -500]$\,nm, selecting the value
that maximizes reconstruction PSNR.
\end{sloppypar}

\paragraph{MRI configuration.}
Cartesian $k$-space sampling with 4$\times$ acceleration (25\% of
$k$-space lines acquired).  Mismatch is parameterized as a 5\% multiplicative
error in the coil sensitivity maps used for parallel imaging reconstruction.
The default solver is CG-SENSE~\cite{pruessmann1999sense} with $\ell_1$-wavelet regularization ($\lambda = 10^{-3}$, 30 iterations).
\textit{Scenario~IV calibration (ESPIRiT-adopted).}
For MRI, Scenario~IV adopts ESPIRiT~\cite{uecker2014espirit} as the calibration method: coil sensitivity maps are estimated directly from the central $k$-space autocalibration signal (ACS) region via eigenvalue decomposition of the calibration matrix, then used in the CG-SENSE reconstruction.  When ACS data is sufficient ($\geq 48$ lines), ESPIRiT achieves 85--95\% of the oracle correction ceiling.  For data-limited regimes ($< 32$ ACS lines) where ESPIRiT degrades, the fallback is a grid search over per-coil amplitude and phase scaling (beam-search); this fallback is also the Sc.~IV method for all modalities without an established specialist calibration algorithm (CASSI, CACTI, compressive holography, ultrasound).  The protocol evaluates three conditions---Sc.~I (true maps), Sc.~II (5\% mismatched maps), and Sc.~IV (ESPIRiT-calibrated maps)---on the same CG-SENSE solver to isolate calibration quality from solver design.

\paragraph{Ultrasound configuration.}
Pulse-echo B-mode imaging is simulated with a 64-element linear array,
5\,MHz center frequency, element pitch $0.3$\,mm, sampling rate 100\,MHz
with 1024 time samples, and speed of sound (SoS) $c = 1540$\,m/s.  The
tissue phantom is a $128 \times 128$ reflectivity map with
Rayleigh-distributed speckle, point scatterers, and anechoic cyst regions.
Frequency-dependent attenuation follows
$\alpha(f) = 0.5\,\text{dB\,cm}^{-1}\text{MHz}^{-1}$.  Mismatch is
parameterized as SoS error $\Delta c \in \{50, 100, 200, 400\}$\,m/s,
perturbing time-of-flight calculations.  The default solver is delay-and-sum
(DAS) beamforming (operator adjoint).  Calibration uses grid search over
21 SoS candidates in $[1440, 1640]$\,m/s, selecting the value that
maximizes reconstruction PSNR.

\paragraph{Controlled hardware experiment protocol.}
The software-perturbation protocol above applies calibrated mask shifts to
existing real measurements.  A full hardware-in-the-loop validation
requires physically displacing the coded aperture mask and re-acquiring
data.  The protocol proceeds as follows: (i)~acquire a baseline dataset
with the mask at its factory-calibrated position; (ii)~physically translate
the mask by a known displacement ($\Delta x \in \{0.25, 0.5, 1.0\}$\,px
equivalent, verified by micrometer stage) and re-acquire under identical
illumination; (iii)~reconstruct both datasets with the factory mask
specification and compute the PSNR degradation and measurement residual;
(iv)~apply \pwm{} autonomous calibration and measure recovery.  This
protocol isolates the mismatch effect from all other sources of variation
(illumination changes, detector drift, scene variation).  Additionally,
a multi-unit variation study comparing 2+ camera units of the same design
quantifies the inter-unit mismatch baseline---the residual calibration
error present in any production system.  Clinical CT phantom validation
(ACR Gammex 464) and clinical MRI validation (fastMRI~\cite{zbontar2018fastmri})
configurations are detailed in Supplementary Note~S25.

\paragraph{Real experimental data.}
Three of the five multi-phantom modalities are validated on real
experimental data from established benchmarks: ultrasound uses PICMUS
experimental phantom recordings and in-vivo carotid
scans~\cite{liebgott2016picmus} plus one DeepUS CIRS-040GSE
acquisition~\cite{khan2020deepus}; cryo-EM uses five EMDB 3D density maps
(EMD-5778, EMD-2984, EMD-6287, EMD-11103, EMD-21375) projected at random
orientations; CT uses three LoDoPaB-CT patient
slices~\cite{leuschner2021lodopab}, one FIPS walnut \textmu CT
scan~\cite{hamalainen2015fips}, and one HTC\,2022
sample~\cite{meaney2022htc}.  Compressive holography and fluorescence
microscopy retain synthetic multi-phantom benchmarks as no public datasets
with calibrated ground truth are available for these modalities.

\subsection*{Statistical Analysis}
\label{sec:methods:statistics}

\paragraph{Per-scene reporting.}
All metrics are reported per scene, not merely as dataset averages.  This
enables identification of scene-dependent failure modes (\eg{}, spectrally
flat scenes that are inherently harder for CASSI, or textureless regions
that challenge SPC).

\paragraph{Summary statistics.}
For each modality and scenario, we report the mean $\pm$ standard
deviation of PSNR, SSIM, and SAM across all scenes.  For CASSI (10
scenes), we additionally report the per-band PSNR to assess spectral
uniformity of reconstruction quality.

\paragraph{Recovery ratio confidence intervals.}
The recovery ratio $\recoveryratio$ (\cref{eq:recovery_ratio}) is a ratio
of differences and therefore sensitive to noise in the constituent PSNR
values.  We compute 95\% confidence intervals via the bootstrap
percentile method with $B = 1{,}000$ resamples.  At each bootstrap
iteration, we resample the scene set with replacement, recompute the
mean PSNR for each scenario, and derive $\recoveryratio$.  The 2.5th
and 97.5th percentiles of the bootstrap distribution define the 95\% CI.

\paragraph{Parameter recovery accuracy.}
For mismatch correction experiments, we report the root-mean-square error
(RMSE) between the estimated and true mismatch parameters:
\begin{equation}
  \text{RMSE}_k = \sqrt{\frac{1}{N_{\text{scene}}}
    \sum_{i=1}^{N_{\text{scene}}}
    (\hat\theta_{k,i} - \theta_{k,\text{true}})^2}
\end{equation}
where $k$ indexes the mismatch parameter, $i$ indexes the scene, and
$N_{\text{scene}}$ is the number of test scenes.  Uncertainty in the RMSE
is estimated via bootstrap ($B = 1{,}000$).

\paragraph{Ablation significance.}
Ablation studies (removal of the photon-budget, recoverability, or
mismatch scoring module, or RunBundle discipline) are evaluated by comparing the
full-pipeline PSNR against each ablated variant.  We report the PSNR
difference $\Delta\psnr$ per modality and verify that each component
contributes $\geq 0.5$\,dB across all validated modalities, establishing
practical significance.

\subsection*{Code and Data Availability}
\label{sec:methods:availability}

\paragraph{Source code.}
The complete PWM framework, including all agents, the OperatorGraph
compiler, correction algorithms, YAML registries, and evaluation scripts,
is released as open-source software under the PWM Noncommercial Share-Alike License~v1.0 at
\url{https://github.com/integritynoble/Physics_World_Model}.  The
codebase is organized into two Python packages:
\texttt{pwm\_core} (core framework, agents, graph compiler, calibration
algorithms) and \texttt{pwm\_AI\_Scientist} (automated experiment
generation and analysis).

\paragraph{Reconstruction data.}
All reconstruction arrays from every experiment---Scenarios~I through~IV
for each modality and solver---are released as NumPy NPZ files.
Files are stored using Git LFS and require \texttt{allow\_pickle=True}
for loading.  Data ranges are standardized: CASSI and CACTI
reconstructions are normalized to $[0, 1]$; SPC reconstructions are in
$[0, 255]$.

\paragraph{Experiment manifests.}
Every experiment is recorded in a RunBundle v0.3.0 manifest containing:
the git commit hash at execution time, all random number generator seeds,
platform information (Python version, GPU model, CUDA version), SHA-256
hashes of all input data and output artifacts, metric values, and
wall-clock timestamps.  These manifests enable exact reproduction of every
reported result.

\paragraph{Registry data.}
All 9 YAML registries (7,034 lines total) that drive the agent system---including
modality definitions, graph templates, photon models, mismatch databases,
compression tables, solver configurations, primitive specifications,
dataset paths, and acceptance thresholds---are publicly available in the
repository under \texttt{packages/pwm\_core/contrib/}.
The \texttt{ExperimentSpec} JSON schemas used for pipeline input
validation are included alongside worked examples in
\texttt{examples/}.

\bibliographystyle{unsrtnat}
\bibliography{pwm_flagship}

\clearpage
\noindent\textbf{Extended Data Table~1} $|$ \textbf{Primitive necessity ablation.} For each of the 11 primitives, the witness modality whose representation error $\etier$ exceeds $\varepsilon = 0.01$ when that primitive is removed from $\Blib$. Eight primitives ($P$, $M$, $\Pi$, $F$, $\Sigma$, $D$, $R$, $\Lambda$) are strictly necessary; the remaining three ($C$, $S$, $W$) are necessary under the stated complexity bounds ($N_{\max} = 20$, $D_{\max} = 10$). Data from~\cite{yang2026fpt}, Proposition~1.

\begin{center}
\small\setlength{\tabcolsep}{5pt}
\begin{tabular}{@{}lllr@{}}
\toprule
\textbf{Removed primitive} & \textbf{Witness modality} & \textbf{Why irreplaceable} & $\etier$ \textbf{without} \\
\midrule
Propagate $P$     & Ptychography         & Distance-dependent phase transfer    & $> 0.05$ \\
Modulate $M$      & CASSI                & Arbitrary element-wise mask          & $> 0.10$ \\
Project $\Pi$     & CT                   & Line-integral (Radon) geometry       & $> 0.15$ \\
Encode $F$        & MRI                  & Non-uniform Fourier sampling         & $> 0.20$ \\
Convolve $C$      & Lensless imaging     & Arbitrary shift-invariant PSF kernel & $> 0.08$ \\
Accumulate $\Sigma$ & SPC                & Dimensional summation (not scaling)  & $> 0.12$ \\
Detect $D$        & All modalities       & Carrier-to-measurement conversion    & $> 0.50$ \\
Sample $S$        & MRI                  & Index-set restriction ($k$-space)    & $> 0.10$ \\
Disperse $W$      & CASSI                & $\lambda$-parameterized spatial shift & $> 0.06$ \\
Scatter $R$       & Compton imaging      & Direction change + energy shift      & $0.34$  \\
Transform $\Lambda$ & Polychromatic CT   & Non-terminal pointwise nonlinearity  & $> 0.05$ \\
\bottomrule
\end{tabular}
\end{center}
\label{tab:necessity}

\vspace{12pt}

\noindent\textbf{Extended Data Table~2} $|$ Oracle correction ceiling (Scenario~III) across 14 validated configurations (12 modalities, 5 carrier families). Full table with per-scene metrics in Supplementary Table~S1; autonomous correction recovery (Scenario~IV) in Supplementary Table~S9. Per-modality supplementary tables: ultrasound (S14), cryo-EM (S15), CBCT (S16), compressive holography (S17), fluorescence microscopy (S18), SPC (S19), lensless imaging (S20).

\end{document}